\documentclass{ecai} 

\usepackage{latexsym}
\usepackage{amssymb}
\usepackage{amsmath}
\usepackage{amsthm}
\usepackage{booktabs}
\usepackage{enumitem}
\usepackage{graphicx}
\usepackage{color}
\allowdisplaybreaks[4]
\usepackage{multirow}
\usepackage{makecell}
\usepackage[caption=false]{subfig}
\newtheorem{theorem}{Theorem}
\newtheorem{lemma}[theorem]{Lemma}

\newtheorem{proposition}[theorem]{Proposition}

\newtheorem{definition}{Definition}

\usepackage[linesnumbered,ruled,vlined]{algorithm2e}

\newcommand{\BibTeX}{B\kern-.05em{\sc i\kern-.025em b}\kern-.08em\TeX}

\begin{document}


\begin{frontmatter}


\paperid{5761} 


\title{Multi-Hop Privacy Propagation for Differentially Private \\ Federated Learning in Social Networks}


\author[A]{\fnms{Chenchen}~\snm{Lin}\orcid{0009-0002-8473-6068}}
\author[A]{\fnms{Xuehe}~\snm{Wang}\orcid{0000-0002-6910-468X}\thanks{Corresponding Author. Email Address: wangxuehe@mail.sysu.edu.cn.}}

\address[A]{School of Artificial Intelligence, Sun Yat-sen University, China.}


\begin{abstract}
Federated learning (FL) enables collaborative model training across decentralized clients without sharing local data, thereby enhancing privacy and facilitating collaboration among clients connected via social networks. However, these social connections introduce privacy externalities: a client’s privacy loss depends not only on its privacy protection strategy but also on the privacy decisions of others, propagated through the network via multi-hop interactions. In this work, we propose a socially-aware privacy-preserving FL mechanism that systematically quantifies indirect privacy leakage through a multi-hop propagation model. We formulate the server-client interaction as a two-stage Stackelberg game, where the server, as the leader, optimizes incentive policies, and clients, as followers, strategically select their privacy budgets, which determine their privacy-preserving levels by controlling the magnitude of added noise. To mitigate information asymmetry in networked privacy estimation, we introduce a mean-field estimator to approximate the average external privacy risk. We theoretically prove the existence and convergence of the fixed point of the mean-field estimator and derive closed-form expressions for the Stackelberg Nash Equilibrium. Despite being designed from a client-centric incentive perspective, our mechanism achieves approximately-optimal social welfare, as revealed by \textit{Price of Anarchy} (PoA) analysis. Experiments on diverse datasets demonstrate that our approach significantly improves client utilities and reduces server costs while maintaining model performance, outperforming both \textit{Social-Agnostic} (SA) baselines and methods that account for social externalities.
\end{abstract}

\end{frontmatter}


\section{Introduction}

\emph{Federated Learning} (FL)~\citep{mcmahan2017communication} enables multiple clients to collaboratively train a shared model without uploading their raw local data. By maintaining data decentralization and localization, FL addresses critical privacy challenges inherent in centralized machine learning. Typical FL algorithms, such as FedAvg~\citep{mcmahan2017communication} and similar algorithms~\citep{acar2021federated,karimireddy2020scaffold,li2020federated} allow clients to train a local model on local datasets and periodically transmit updates to the central server, which aggregates these updates to reformulate a new global model.

While FL prevents direct exposure of raw data, it does not eliminate privacy risks, as model updates may inadvertently encode sensitive information about local datasets. For instance, gradient inversion attacks can reconstruct data from shared gradients~\citep{geiping2020inverting, guo2025new}, and membership inference attacks can identify whether specific data samples were included in the training set by analyzing outputs~\citep{arevalo2024task,hu2021source,xu2024robust}.
To further enhance the privacy preservation of FL against such attacks, the \emph{local differential privacy} (LDP)~\citep{dwork2014algorithmic} mechanism has emerged as a widely adopted approach for its ease of implementation and rigorous quantification of privacy loss. A growing quantity of research has explored enhancing privacy in FL by incorporating the LDP mechanism~\citep{chen2024personalized,truex2020ldp,wei2021user}, which perturbs local updates by injecting noise before model transmission to provide privacy guarantees.

However, most existing privacy-preserving mechanisms in FL assume independent client operations, without potential inter-client correlations. In practice, clients are generally connected via social networks, where their correlations reflect shared behaviors, preferences, and data distributions~\citep{yu2024correlation}. These social connections facilitate richer collaboration among clients, enabling more personalized and efficient FL training, but simultaneously, introduce \emph{external privacy risks}~\citep{yang2019socially}. Specifically, clients' privacy risk may be determined by their actions and indirectly by the behaviors of their socially connected peers via multi-hop propagation across social networks~\citep{liu2014assessment,nguyen2017probability}~\footnote{For example, in a federated recommendation system deployed over a social platform, certain clients' preferences may inadvertently reveal sensitive interests shared with their direct friends, such as similar purchase habits. In healthcare, infectious or genetic conditions in one individual can reveal sensitive health information about socially or biologically related individuals.}. If a client's model is compromised, the risk of information leakage can extend to their immediate neighbors and further propagate to second-hop users, such as friends of friends who share overlapping group memberships or activity patterns. Through such propagation, even clients with strong local privacy settings can still suffer privacy leakage due to the strategies of distant acquaintances within social networks. Thus, the privacy externality presents fundamental challenges for achieving robust privacy protection in FL over social networks. This leads to our first key question in this work:

\textit{\textbf{Key Question 1}: How to model multi-hop privacy propagation dynamics and quantify external privacy risks in social-connected FL?}

Another critical aspect of privacy-preserving FL lies in the unrealistic assumption commonly made in the existing literature: once invited by the central server, clients will participate in training unconditionally. This assumption overlooks the inherent training costs, including computational/communication overhead and the risk of privacy leakage, etc~\citep {yu2020fairness}, which are further exacerbated by \textit{external privacy risks} from networked externalities in social networks. Without a carefully designed economic incentive mechanism, egocentric clients may be unwilling to contribute their local updates to the FL systems. 

Furthermore, even if clients choose to participate, there exists a trade-off between privacy protection and model performance. On the one hand, each client hopes to enhance privacy guarantees to improve the privacy preservation level by injecting larger noise through the LDP mechanism, while it inevitably degrades the global model performance, and receives less reward from the central server~\citep{mao2024game}. On the other hand, the central server tends to gain better model performance at less cost. However,  a lower payment leads to a lower willingness of clients to sacrifice their privacy, leading to stronger privacy settings and causing worse model performance~\citep{xu2021incentive}. Balancing such trade-offs becomes even more challenging in social networks, where inter-client dependencies exaggerate privacy risks. This motivates us to study the second key question in this work:

\textit{\textbf{Key Question 2}: How to coordinate server-client interactions over social networks that account for both local privacy preservation and global model training objectives?}

Given these considerations, we innovatively propose a socially-aware privacy-preserving incentive mechanism in FL that quantifies indirect privacy leakage by the multi-hop privacy propagation model. The multi-fold contributions are summarized as follows:
\begin{itemize}
    \item \textbf{Dynamic Modeling of Multi-hop Privacy Propagation in FL:} We develop a dynamic model that characterizes multi-hop privacy propagation in FL systems over social networks, capturing how clients incur external privacy risks through multi-hop social connections influenced by others’ privacy decisions. This formulation incorporates both temporal dynamics and social topologies, offering a more realistic representation of privacy risk evolution compared to existing static or local models.


    \item \textbf{Game-Theoretic Formulation of Socially-Aware FL:} 
    We propose MPPFL, the first game-theoretic framework of privacy-preserving FL under multi-hop privacy propagation over social networks. Modeling the interaction between the central server and clients as a two-stage Stackelberg game, we derive the server’s optimal reward strategy in Stage \uppercase\expandafter{\romannumeral1} and each client’s optimal dynamic privacy budget in Stage \uppercase\expandafter{\romannumeral2} across global iterations. We theoretically demonstrate that the optimal strategy profile constitutes a Stackelberg Nash Equilibrium, ensuring strategic stability under external privacy risks through social networks.

    \item \textbf{Mean-field Estimator and Efficiency Analysis for Clients' Strategy:} We introduce a mean-field estimator to approximate external privacy risks under incomplete network information and prove the existence and uniqueness of its fixed point. Further, we quantify system efficiency via the PoA and demonstrate that MPPFL achieves approximately optimal social welfare compared to the SA strategy. Experiments on real-world datasets validate the effectiveness of MPPFL on utilities and model performance.

    

\end{itemize}

\subsection{Related Work}
\paragraph{Differential Privacy Incentives for FL.} Several works have studied incentive mechanisms with \textit{Differential Privacy} (DP) in FL, aiming to balance privacy protection and utility. Xu~et al.~\citep{xu2021incentive} proposed NICE, a DP-based incentive mechanism where clients decide privacy budgets to add Laplace noise to loss functions, and both server and clients maximize their utilities. Wang~et al.~\citep{wang2023trade} propose a dynamic privacy pricing game that allows clients to reduce noise in exchange for the server's payments, enhancing FL model utility. Mao~et al.~\citep{mao2024game} modeled privacy-preserving behaviors in cross-silo FL, improved social efficiency, and reduced client costs compared to the subgame perfect Nash equilibrium. Other works have approached the problem using Auction~\citep{ren2023differentially} and Contract mechanism~\citep{wu2021incentivizing}. However, these methods assume that client decisions are independent, which neglects the external privacy risks introduced by social connections.

\paragraph{FL over Social Networks.}Recent studies have extended FL to social networks, using social networks for personalized and collaborative training. Khan et al.~\citep{khan2021socially} proposed DDFL to form efficient clusters to enhance communication and robustness by leveraging social similarity, edge betweenness, and physical proximity. Hu et al.~\citep{hu2024federated} proposed GCAFM to capture user preferences and feature interactions while ensuring data privacy, which integrates graph convolutional autoencoders and factorization machines. He~et~al.~\citep{he2019central} considered FL in decentralized social networks with unidirectional trust. While these works utilize social structure to enhance performance, they largely overlook privacy risks introduced by social connections. Only recently have researchers started to investigate such social-induced privacy risks in FL~\citep{lin2021friend}. For example, Sun~et al.~\citep{sun2024socially} proposed SARDA, a socially-aware iterative double auction mechanism, which accounts for clients' privacy risks from their friends’ participation in FL due to data correlations. However, SARDA captures only direct social influence, ignoring the impact of multi-hop privacy propagation over social networks.

In contrast to the above works, our approach explicitly considers the propagation of privacy leakage over social networks and designs a socially-aware incentive mechanism for balancing the trade-off between local privacy preservation and global training performance.






\section{System Model and Problem Formulation}
\subsection{Standard Federated Learning Model} 
Under the standard FL framework, we assume that $N$ clients participate in FL, and each client $x_i$, $i \!\in\! \{1, 2, \ldots, N\}$ uses private and local data $\mathcal{D}_{i}$ with datasize $|\mathcal{D}_{i}|$ to train its local model. At each global iteration $t \!\in\! \{0,1, \ldots, T\}$, client $x_i$ updates its local model parameter parallelly with \emph{gradient descent} by $\boldsymbol{w_{i}(t+1)}=\boldsymbol{w(t)}-\eta \nabla F_{i}(\boldsymbol{w(t)})$,
in which $\eta$ represents the learning rate and $\nabla F_{i}(\boldsymbol{w(t)})$ is client $x_i$'s gradient at iteration $t$. Once $N$ clients send back the local parameters, the central server aggregates the local models to obtain the updated global parameter by $\boldsymbol{w(t)} \!=\! \sum_{i=1}^{N} \theta_{i} \boldsymbol{w_{i}(t)} \!=\! \sum_{i=1}^{N} \frac{|\mathcal{D}_{i}|}{\sum_{j=1}^{N} |\mathcal{D}_{j}|} \boldsymbol{w_{i}(t)}$. Then, the central server dispatches the updated global parameter to all clients for the next iteration. The goal of the FL method is to find the optimal global parameter $\boldsymbol{w}^{*}$ to minimize the global loss function $F(\boldsymbol{w})$, i.e., $\boldsymbol{w}^{*} \!=\! \text{arg}\min _{\boldsymbol{w}} F(\boldsymbol{w}) \!=\! \text{arg}\min _{\boldsymbol{w}}  \sum_{i=1}^{N} \theta_{i} F_{i}(\boldsymbol{w})$.

\subsection{$\rho$-$z$CDP Mechanism}
To protect against gradient inversion attacks while enabling iterative training in FL, we adopt the $\rho$\textit{-zero-Concentrated Differential Privacy} ($\rho$-$z$CDP) mechanism~\citep{bun2016concentrated}. In our setting, each client perturbs its local parameter with additive Gaussian noise before transmitting to the server. Specifically, at global iteration $t$, the local parameter of client $x_i$ is perturbed as $\nabla \widetilde{F}_i(\boldsymbol{w}(t)) = \nabla F_i(\boldsymbol{w}(t)) + \boldsymbol{n}_i(t)$, where Gaussian noise $\boldsymbol{n}_i(t) \sim \mathcal{N}(0, \delta_i^2(t)\boldsymbol{I}_p)$ with covariance $\delta_i^2(t) \boldsymbol{I}_p$. Further, the variance of the injected noise is analytically determined by the client’s privacy budget $\rho_{i}(t)$ of each client $x_i$ at $t$-th iteration. For a query function $Q$ with $\ell_2$-sensitivity $\Delta Q$, the Gaussian mechanism satisfies $\rho$-$z$CDP with noise variance, i.e., $\delta_i^2(t) = \frac{\Delta Q^2}{2\rho_i(t)}$. \emph{The privacy preserving level is determined by $\rho_i(t)$, where a lower $\rho_{i}(t)$ requires injecting larger noise and vice versa.} Then, we can derive the upper bound of the $\ell_2$-sensitivity $\Delta Q$ and further quantify the artificial Gaussian noise variance $\delta_i^2(t)$ as follows.
\begin{proposition} \label{proposition_varience}
At $t$-th iteration of global aggregation, the local model parameters of client $x_i$ are perturbed based on $\rho$-$z$CDP mechanism and Gaussian noise, leading to the noise variance $\delta_i^2(t)$ given by $\delta_i^2(t) \!=\! \frac{2\mathcal{S}^2}{|\mathcal{D}_i|^2\rho_{i}(t)}$, where $\mathcal{S}$ denotes the threshold.
\end{proposition}

\subsection{Multi-hop Privacy Propagation Model}
Conventional DP-based FL primarily designs privacy budgets from self-assessed privacy risks. However, in practice, clients' privacy exposure is also influenced by the indirect effects of other clients’ privacy strategies within social networks. In particular, social relationships facilitate private information diffusion via multi-hop propagation paths, leading to unintended privacy risks. For instance, clients with weaker privacy settings, even if not directly connected to target clients, may still conduct impacts through intermediate relay nodes within social networks. Thus, assessing clients' privacy risk requires accounting for direct interactions and cumulative multi-hop effects.

Motivated by the above observations, we propose a \emph{Multi-hop Privacy Propagation} model to characterize \textit{external privacy risks} by incorporating indirect influences within social networks. We model the social connections among $N$ participating clients as a directed graph $\boldsymbol{G} \!=\! (\boldsymbol{X}, \boldsymbol{W})$, where $\boldsymbol{X} = \{x_1, x_2, \dots, x_N\}$ represents the client set, and $\omega_{ij} \in \boldsymbol{W}$ indicates the strength of social interaction from clients $x_j$ to $x_i$, with $\omega_{ij} \!\in\! (0,1)$ for any $i,j \in \{1, 2, \ldots, N\}$. Without loss of generality, we exclude isolated clients and assume no self-loops, i.e., $\omega_{ii}=0$. To quantify the impact of clients on individual privacy exposure, we perform row-wise normalization over $\boldsymbol{W}$, yielding a row-stochastic matrix $\boldsymbol{\widetilde{W}}$, where $\tilde{\omega}_{ij} \!=\! \frac{\omega_{ij}}{\sum\nolimits_{q=1}^{N}\omega_{iq}}$ \footnote{We introduce a lower bound $\tilde{\omega}_{\min} \leq \tilde{\omega}_{ij}$ to reflect the realistic assumption that social influence among connected clients is never entirely negligible, as even weakly connected individuals, e.g., occasional interactions or shared group membership, lead to privacy risk propagation in real-world settings.}, which aligns with \citep{he2019central} and is particularly suitable for FL where clients may have limited knowledge of the overall network structure.


Note that clients' external privacy risks are influenced by direct neighbors and by indirect privacy risks propagated through multi-hop connections. Specifically, clients $x_i$ and $x_j$ share a direct edge (i.e., $1$-hop connection), and the privacy decisions of client $x_j$ will directly impact the privacy exposure of client $x_i$. More generally, for any positive integer $K \!>\! 1$, the $K$-hop propagation accounts for privacy risks originating from neighbors of neighbors and even more distant clients. This cascading effect indicates that clients' privacy exposure is shaped by the collective decisions of directly and indirectly connected clients across social networks. To formally quantify the cumulative effects, we define the external privacy risk coefficient by $\sigma_{ij} \!=\! \sum\nolimits_{k=1}^{K}\lambda^{k-1} (\boldsymbol{\widetilde{W}}^{k})_{ij}$, where $\boldsymbol{\widetilde{W}}^{k}$ represents the $k$-th power of the matrix $\boldsymbol{\widetilde{W}}$, and $(\boldsymbol{\widetilde{W}}^{k})_{ij}$ denotes the impact strength with path length $k$ from $x_j$ to $x_i$ \footnote{We explicitly set $\sigma_{ii} = 0$ to exclude self-influence, as a client’s external privacy leakage risk stems solely from interactions with other clients in the social networks. This design aligns with the formulation of the client's composite privacy leakage risk that internal privacy loss is independently addressed by the LDP mechanism, while $\sigma_{ij}$ focuses exclusively on external privacy risks.}. The discount factor $\lambda \in (0, 1)$ regulates the attenuation of privacy leakage risk with increasing path length, ensuring the impact weakens as the number of hops grows. 



Then, we formally define the overall external privacy risks incurred by client $x_i$ at iteration $t$. Given that privacy leakage is influenced by multi-hop social interactions, the risk depends not only on their own privacy decisions but also on the privacy budgets chosen by other clients in social networks. We denote the privacy budget set of all clients by $\boldsymbol{\rho}(t) \!=\! \{\rho_{i}(t)$, $i \!\in\! \{1, 2, \ldots, N\}\}$, $t \!\in\! \{0,1, \ldots, T\}\}$. Denote the external privacy risks of client $x_i$ at iteration $t$ by $\mathcal{R}_{i}(t)$, which is the weighted sum of all clients’ privacy budgets, with weights determined by the external privacy leakage risk coefficients by $\mathcal{R}_{i}(t) = \sum\nolimits_{j=1}^{N} \sigma_{ij} \rho_{j}(t)$.
It quantifies the cumulative external privacy risks that client $x_i$ faces from other clients' privacy decisions within social networks. Thus, higher privacy budgets chosen by neighbors or accessible clients through multi-hop social ties will proportionally increase the external privacy risks incurred by client~$x_i$.

\begin{figure}[t]
\centering
\includegraphics[width=0.495\textwidth]{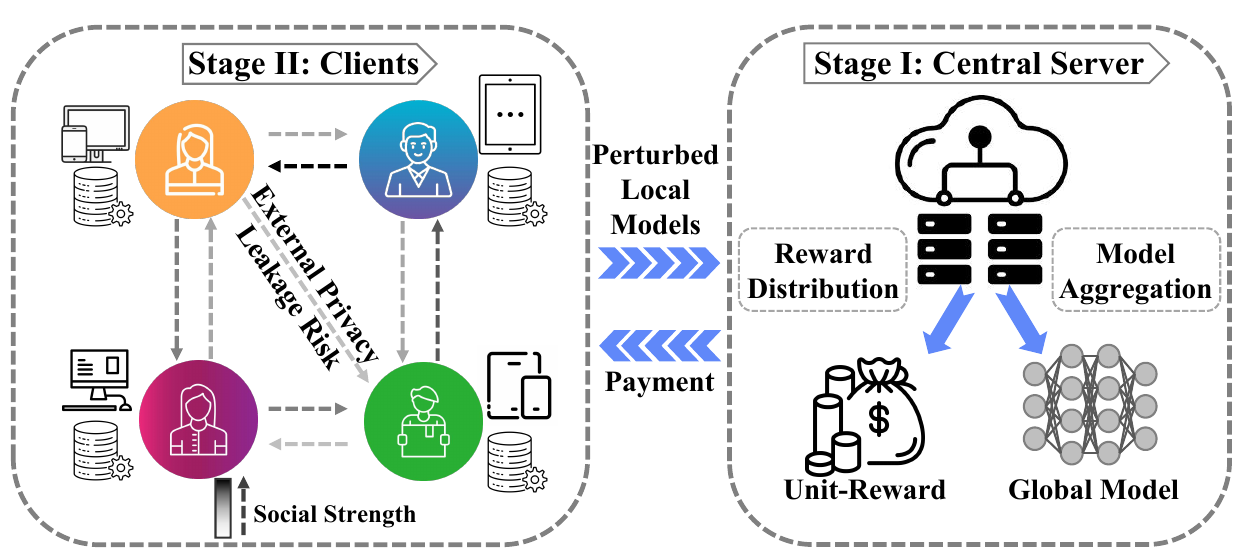}
\vspace{-20pt}
\caption{MPPFL: A socially-aware privacy-preserving incentive mechanism in FL, where the left dashed outline represents the process of local training, model perturbation, and uploading with the external privacy risks over social networks, and the right dashed outline represents the model aggregation and reward distribution by the central server.}
\label{framework}
\vspace{5pt}
\end{figure}

\subsection{Problem Formulation}
We formulate the interactions between the central server and clients as a two-stage Stackelberg game-based joint optimization problem, as shown in Figure~\ref{framework}. At each iteration, the central server \emph{(leader)} announces the unit reward, denoted by $r(t)$, to influence clients’ privacy strategy and minimize its cost $U_{t}$. In response, each client \emph{(follower)} independently chooses the privacy budget $\rho_{i}(t)$ to maximize its utility $U_{i}$, considering the trade-off between received rewards and training costs. Notably, clients incur additional privacy costs from external privacy risks propagated via multi-hop social interactions.


\paragraph{Central Server's Cost Design in Stage \uppercase\expandafter{\romannumeral1}.} The central server's cost function can be defined as the model performance loss, which relates to the accuracy loss of the global model and the reward allocated to clients. Since the client-side privacy budget vector $\boldsymbol{\rho} \!=\! \{\boldsymbol{\rho}(t), t \!\in\! \{0,\\ 1, \ldots, T\}\}$, $\boldsymbol{\rho}(t) \!=\! \{\rho_{i}(t), i \!\in\! \{1, 2, \ldots, N\}\}$ impacts the magnitude of noise injected into local parameters, degrading the global model accuracy. Suppose that the gradient norm at iteration $t$, denoted by $\|\nabla F(\boldsymbol{w}(t))\|_2$, is upper bounded by a positive constant $\mathcal{E}$, similarly in~\citep{ding2023incentive, yuan2025game}. Then, we analyze the accuracy loss of the global model with clients' privacy budgets in the following proposition.
\begin{proposition}[Accuracy Loss of Global Model] \label{global_accuracy_loss} 
With the optimal global model parameter $\boldsymbol{w^{*}}$, the accuracy loss of the global model at $t$-th iteration is upper bounded by:
\begin{align} \label{accuracy_loss}
\mathbb{E}[F(\boldsymbol{w(t)}) - F(\boldsymbol{w^{*}})] \leq \frac{\beta \mathcal{E}^{2}}{2\mu^{2}t} + \sum\nolimits_{i=1}^{N} \frac{\epsilon_{i}}{t \rho_{i}(t)},
\end{align}
where the personalized hyperparameter $\epsilon_{i} = \frac{ p \beta \mathcal{S}^{2}\theta_{i}^{2}}{\mu^{2} |\mathcal{D}_{i}|^{2}}$ and $p$ denotes the dimension of model parameters.
\end{proposition}

Proposition \ref{global_accuracy_loss} highlights the impact of factors, i.e., local aggregation weight $\theta_i$, datasize $|\mathcal{D}_{i}|$, and privacy budgets $\rho_{i}(t)$, on the global model’s accuracy. To specifically capture the influence of client-related factors on global model performance, we define a metric $\sum\nolimits_{i=1}^{N} \frac{\epsilon_{i}}{t \rho_{i}(t)}$. Meanwhile, the total reward paid by the central server to all clients at the $t$-th iteration is expressed as $\sum\nolimits_{i=1}^{N} r(t)\rho_{i}(t)$. Thus, the cost function of the central server is defined as follows:
\begin{align}\label{central_server_cost}
\text{\!\!\!\footnotesize$U_{t}(r(t), \boldsymbol{\rho}(t)) \!=\!\tau \! 
\sum\nolimits_{i=1}^{N}\! \frac{\epsilon_{i}}{t \rho_{i}(t)} \!+\! (1 \!-\! \tau)\! \sum\nolimits_{i=1}^{N}\! r(t)\rho_{i}(t)$}, \!\!\!
\end{align}
where weight $\tau \in (0, 1)$ balances the trade-off between the accuracy loss and payment term.

\paragraph{Clients’ Utility Design in Stage \uppercase\expandafter{\romannumeral2}.} To incentivize client participation while accounting for privacy preferences, we define the reward received by client $x_{i}$ at iteration $t$ by $r(t)\rho_{i}(t)$. Specifically, clients with higher privacy budgets introduce less perturbation, thereby receiving greater rewards and promoting model performance. Conversely, strict privacy constraints result in lower rewards, reflecting the trade-offs between privacy preservation and contribution quality. To quantify the privacy cost of client $x_{i}$ at iteration $t$, which corresponds to the compensation required for tolerating privacy leakage, we follow a widely adopted approach in previous studies~\citep{zhan2019free, xu2023aoi} and model clients' costs by utilizing convex quadratic functions with coefficients $a_{i},b_{i} > 0$, formulated as $a_{i}s_{i}^{2}(t) + b_{i}s_{i}(t)$, where $s_{i}(t) = \\ \rho_i(t)+\alpha\mathcal{R}_{i}(t)$ denotes the composite privacy leakage risks, discount factor $\alpha \!\in\! (0, 1)$. Specifically, $s_{i}(t)$ comprises two components, i.e., \textit{internal privacy risks}, which arises from noise injected into the local model governed by privacy budget $\rho_i(t)$, and \textit{external privacy risks}, which stems from information exposure due to $K$-hop privacy propagation in social networks, captured by $\mathcal{R}_{i}(t)$. Thus, the utility function of each client is defined as the difference between the reward gained from the central server and the total costs at $t$-th iteration:
\begin{align}\label{client_utility}
\text{\!\!\!\!\!\!\footnotesize$U_{i}(r(t), \rho_{i}(t), \boldsymbol{\rho}_{\text{-}i}(t)) \!=\! r(t)\rho_{i}(t) \!-\! C_i \!-\! ( a_{i}s_{i}^{2}(t) \!+\! b_{i}s_{i}(t))$}, \!\!\!\!\!\!
\end{align}
where $\boldsymbol{\rho}_{\text{-}i}(t) \!=\! \{\rho_{1}(t),\ldots, \rho_{i\text{-}1}(t), \rho_{i\text{+}1}(t), \ldots,\rho_{N}(t)\}$. $C_i \!=\! \kappa_{i}\xi_{i}f_{i}|\mathcal{D}_{i}|L$ denotes the computation cost of client $x_{i}$ to complete the computation task similar to~\citep{ng2020joint,zhan2020incentive}, where $\kappa_i$, $\xi_i$, and $f_i$ denote the CPU architecture factor, CPU cycles, and clock frequency.

\section{Stackelberg Nash Equilibrium Analysis}

\subsection{Determining the Optimal Strategy Group}
At iteration $t$, the privacy budget $\rho_{i}(t)$ of client $x_i$ is inherently coupled with budgets $\rho_{j}(t)$ of other clients, where $j \neq i$, as dictated by the dynamic and nonlinear optimization problem in Eq.~(\ref{client_utility}). However, clients operate in isolation and do not exchange private information during local training, which renders it impractical for any client to acquire the real-time strategies of others, thereby complicating the design of jointly optimal strategies across clients over time. To overcome this challenge, we adopt the \textit{mean-field approximation}, where each client uses estimated average external privacy risks, including privacy budgets of other clients, to make decentralized decisions without accessing any private information from others. 

\begin{definition}[Mean-Field Estimator]\label{mean_field} 
To obtain the optimal privacy budget for each client, a mean-field estimator $\phi_{i}(t)$ is introduced to approximate the average external privacy risks:
\begin{align}
\phi_{i}(t) =\frac{1}{N} \mathcal{R}_{i}(t) = \frac{1}{N} \sum\nolimits_{j = 1}^{N} \sigma_{ij} \rho_{j}(t).
\end{align}
\end{definition}

From the mathematical point of view, the mean-field estimator $\phi_{i}(t)$ is a given function in our FL optimization problem. Thus, the objective function of client $x_i$ in Eq.~(\ref{client_utility}) can be further refined as:
\begin{align}\label{client_utility_transform}
\hat{U}_{i}(r(t), \rho_{i}(t), \boldsymbol{\rho}_{\text{-}i}(t)) \!=&\ r(t)\rho_{i}(t) \!-\! C_i \!-\! [a_{i}(\rho_{i}(t) +\! \alpha N \phi_{i}(t))^{2}   \nonumber \\
&\!+\! b_{i}(\rho_{i}(t) \!+\! \alpha N \phi_{i}(t)) ] . \!\!
\end{align}

To determine the optimal strategy profile $(r^{*}(t), \boldsymbol{\rho}^{*}(t))$, we adopt the backward induction approach and begin by analyzing Stage \uppercase\expandafter{\romannumeral2}, where each client determines its optimal privacy budget $\rho^{*}_{i}(t)$ by maximizing $\hat{U}_{i}$ for any given unit-reward $r(t)$. In Stage \uppercase\expandafter{\romannumeral1}, the central server decides the optimal unit-reward $r^{*}(t)$ by minimizing $U_{t}$, considering clients’ best responses. Then, the closed-form optimal privacy budget $\rho_i^*(t)$ is summarized in the following theorem.
 
\begin{theorem}[Optimal Privacy Budget]\label{optimal_rho} 
In Stage \uppercase\expandafter{\romannumeral2}, given the mean-field estimator $\phi_{i}(t)$ of each client over time, the privacy budget of client $x_i$, $i \!\in\! \{1, 2, \ldots, N\}$ at iteration $t \!\in\! \{0,1, \ldots, T\}$ is given~by:
\begin{align}\label{optimal_privacy_budget}
\rho_i^*(t) =\frac{r(t) - b_{i}}{2a_{i}} - \alpha N \phi_i(t).
\end{align}
\end{theorem}

\begin{proof}[Proof Sketch]
By analyzing the first and second-order derivatives of the utility function $\hat{U}_{i}$ in Eq.~(\ref{client_utility_transform}), we demonstrate that $\hat{U}_{i}$ is strictly concave in $\rho_{i}(t)$ as the second-order derivative is negative, which guarantees the unique optimal privacy budget $\rho_i^*(t)$ in closed-form by setting the first-order derivative to zero. Therefore, the client’s best strategy can be efficiently derived as shown in Eq.~(\ref{optimal_privacy_budget}).
\end{proof}


Building on the optimal client-side privacy budget $\rho^{*}_{i}(t)$ in Theorem~\ref{optimal_rho}, the optimal unit-reward $r^{*}(t)$ for the central server in Stage~\uppercase\expandafter{\romannumeral1} can be obtained by substituting Eq.~(\ref{optimal_privacy_budget}) into the server’s cost function in Eq.~(\ref{central_server_cost}), therefore formulating the following lemma.
\begin{lemma}[Optimal Unit-Reward]\label{optimal_unit_reward_lemma}
In Stage~\uppercase\expandafter{\romannumeral1}, according to the optimal privacy budget $\rho^{*}_{i}(t)$ in Theorem~\ref{optimal_rho}, the optimal unit-reward $r^{*}(t)$ of the central server at iteration $t \in {0,1,\ldots,T}$ is derived as:
\begin{align}\label{optimal_unit_reward}
\text{\!\!\!\!\!\! \footnotesize $r^{*}(t) \!=\! \frac{\tau}{1 \!-\! \tau}\! \sum\nolimits_{i=1}^{N}  [ \frac{2 a_{i} \epsilon_{i} / (\sum\nolimits_{i=1}^{N}a^{-1}_{i}) }{t (r(t) - b_{i}- 2a_{i} \alpha N \phi_{i}(t))^2} ]  \!+\! \frac{H(t)}{\sum\nolimits_{i=1}^{N}a^{-1}_{i}}, $ \!\!\!\!}
\end{align}
where $H(t)= \sum\nolimits_{i=1}^{N}[ \frac{b_{i}}{2a_{i}} + \alpha N \phi_{i}(t) ]$.
\end{lemma}

\begin{algorithm}[t]
\small 
\caption{Iterative Calculation of Fixed Point $\phi_i(t)$} 
\label{alg_fix_point}
\SetAlgoLined
\KwIn{$N$, $T$, edge set $\boldsymbol{W}$, the threshold $\epsilon_{0}$, $m \!=\! 1$, arbitrary initial value for $\epsilon$, $\rho_{i}(t)$, $\phi_{i}(t)$, $r(t)$, where $t \!\in\! \{0,1, \ldots, T\}$, $i \!\in\! \{1,2,\ldots,N\}$.}
\KwOut{The fixed point of $\phi_i(t)$, $t \!\in\! \{0,1, \ldots, T\}$, $i \!\in\! \{1,2,\ldots,N\}$.}
\For{$t \rightarrow 0, 1, \ldots, T$}{
    \While{$\epsilon > \epsilon_{0}$}{
        \For{$i \rightarrow 1, 2, \ldots, N$}{
            Update $\rho_{i}^m(t)$ according to Eq.~(\ref{optimal_privacy_budget}) \;
        }
        \For{$t \rightarrow 0, 1, \ldots, T-1$}{
            Update $r^m(t)$ according to Eq.~(\ref{optimal_unit_reward}) and $\rho_{i}^m(t)$\;
        }
        $\phi_{i}^{m,\texttt{est}}(t) =  \frac{1}{N}\sum\nolimits_{i=1}^{N} \sigma_{ij} \times \rho_{j}(t)$ \;
        $\phi_i(t) = \phi_{i}^{m,\texttt{est}}(t)$, $r(t) = r^{m}(t)$,$\rho_i(t)=\rho_{i}^m(t)$ \;
        $\epsilon = \phi_{i}^{m,\texttt{est}}(t) - \phi_{i}^{m-1,\texttt{est}}(t)$ \;
        $m = m + 1$\;
    }
}
\end{algorithm}

Eq.~(\ref{optimal_unit_reward}) characterizes a high-order, non-polynomial equation, rendering the closed-form solution intractable. Nevertheless, the optimal unit-reward $r^*(t)$ can be efficiently obtained by existing and widely adopted numerical optimization techniques, e.g., Gradient Descent or Quasi-Newton methods, ensuring convergence under strict convexity. Next, we establish the theoretical foundations for analyzing the Stackelberg Nash Equilibrium of our MPPFL framework.
\begin{definition}[Stackelberg Nash Equilibrium, SNE]
The optimal strategy profile $(r^{*}(t), \boldsymbol{\rho}^{*}(t))$ at iteration $t \!\in\! \{0, 1, \ldots, T\}$ constitutes a Stackelberg Nash Equilibrium if for any unit-reward $r(t) \in \mathbb{R}$ and any privacy budget $\rho_i(t) \in \mathbb{R}$ of any client $x_i$, $i \!\in\! \{1,2,\ldots,N\}$,
\begin{align}
U_{t}(r^{*}(t),\boldsymbol{\rho}^{*}(t)) &\leq U_{t}(r(t),\boldsymbol{\rho}^{*}(t)),  \\
U_{i}(\boldsymbol{\rho_{i}}^{*}(t), \boldsymbol{\rho_{\text{-}i}}^{*}(t), r^{*}(t)) &\geq U_{i}(\boldsymbol{\rho_{i}}(t), \boldsymbol{\rho_{\text{-}i}}^{*}(t), r^{*}(t)). 
\end{align}
\end{definition}
\begin{theorem}[SNE Existence]
The above two-stage Stackelberg game possesses a Stackelberg Nash Equilibrium.
\end{theorem}
\begin{proof}
Based on Theorem \ref{optimal_rho}, we first derive the optimal privacy budget $\rho_{i}(t)$ for each client $x_i$ in response to the given unit-reward $r(t) \in \mathbb{R}$. Subsequently, we demonstrate that there exists an optimal unit-reward $r^{*}(t)$ for the central server’s cost function $U_{t}$ according to the optimal privacy budget $\rho_{i}^{*}(t)$ in Theorem \ref{optimal_rho}. As shown in Lemma \ref{optimal_unit_reward_lemma}, the cost function of the central server is strictly convex in $r(t)$ and the derivative satisfies $\lim_{r(t) \rightarrow 0} \frac{\partial U_{t}}{ \partial r(t)} \!\rightarrow\! - \infty$ and $\lim_{r(t) \rightarrow +\infty} \frac{\partial U_{t}}{ \partial r(t)} \!\rightarrow\! + \infty$, guaranteeing the unique minimum. Then, the optimal unit-reward $r^{*}(t)$ can be obtained by solving the first-order derivation $\frac{\partial U_{t}(r(t),\boldsymbol{\rho}^*(t))}{\partial r(t)} = 0$. Therefore, the strategies $r^{*}(t)$ and $\boldsymbol{\rho}^*(t)$ constitute a mutual optimal strategy between the central server and the clients at each iteration $t$ and the optimal strategy profile $(r^{*}(t), \boldsymbol{\rho}^{*}(t))$ forms the Stackelberg Nash Equilibrium. \end{proof}

\begin{figure*}
\setlength{\abovecaptionskip}{2pt} 
    \centering
    \begin{minipage}{160pt}
        \includegraphics[width=1.0\textwidth, trim=10 10 10 0,clip]{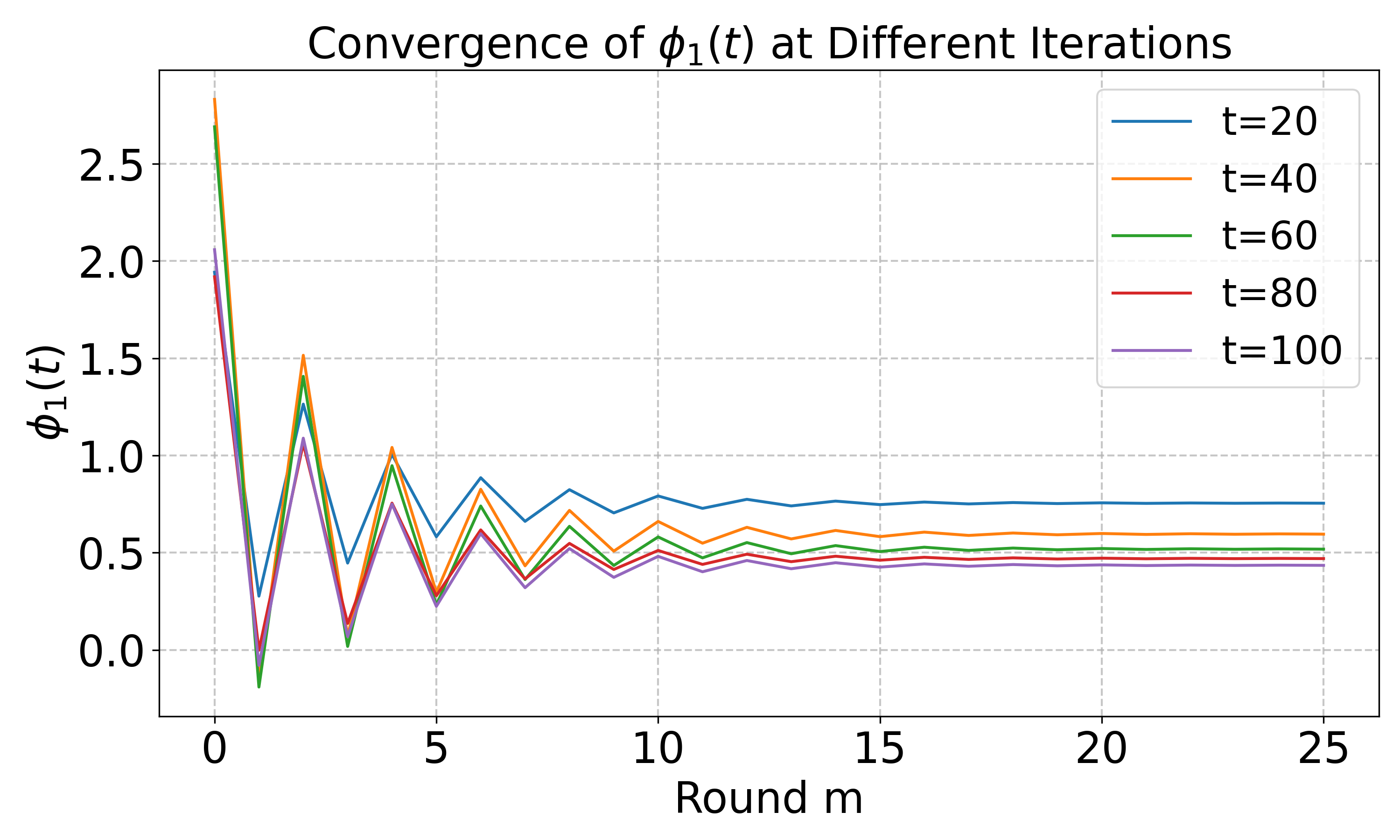}
        \vspace{-20pt}
        \caption{The convergence for the mean-field estimator $\phi_1(t)$ at different iteration $t$.}
        \label{phi_all_t_convergence}
    \end{minipage}
    \hspace{2pt}
    \begin{minipage}{160pt}
        \includegraphics[width=1.0\textwidth, trim=35 40 30 35,clip]{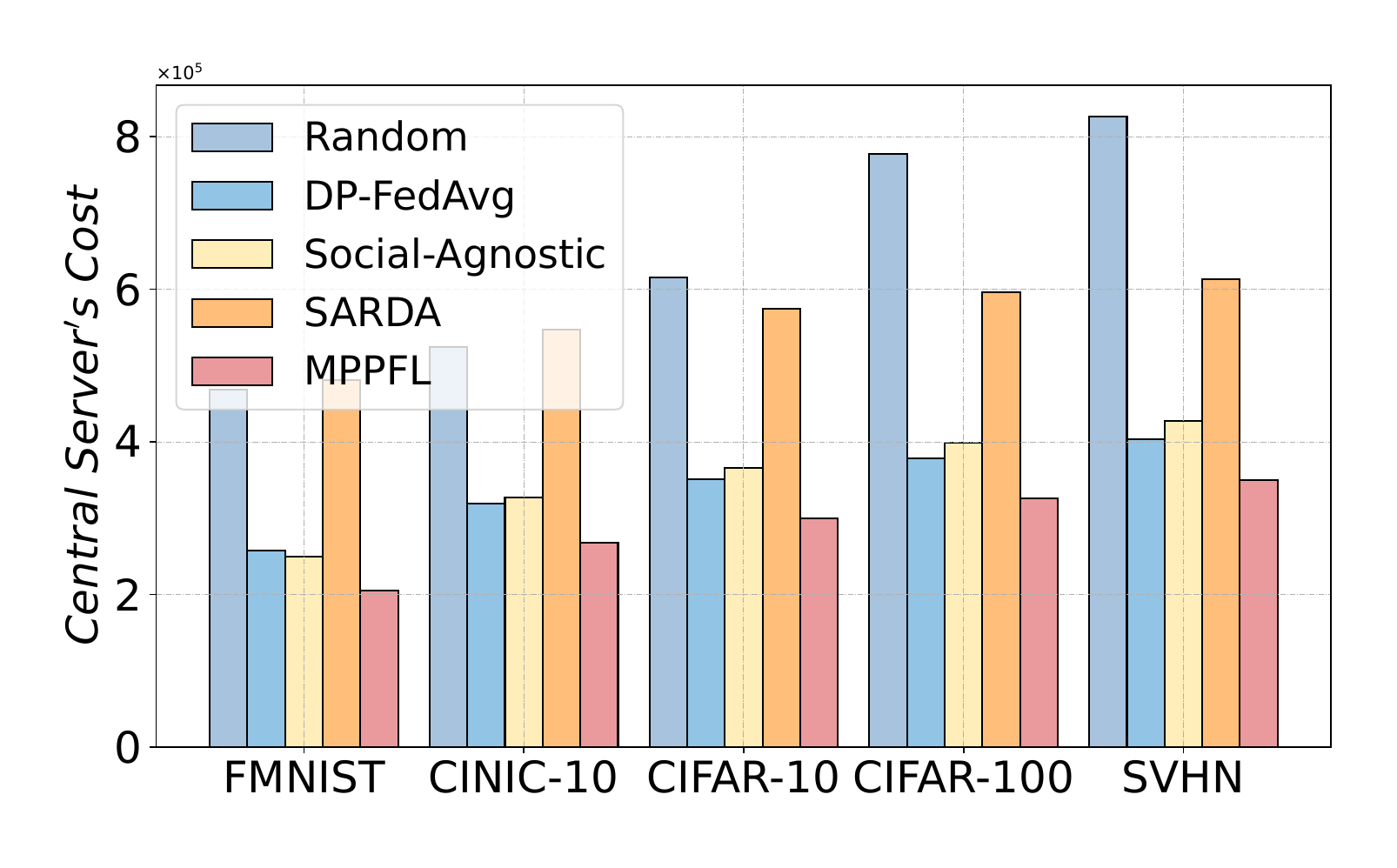}
        \vspace{-20pt}
        \caption{The central server’s cost comparison.}
        \label{scs_compare}
    \end{minipage}
    \hspace{2pt}
    \begin{minipage}{160pt}
        \includegraphics[width=1.0\textwidth, trim=35 40 30 35,clip]{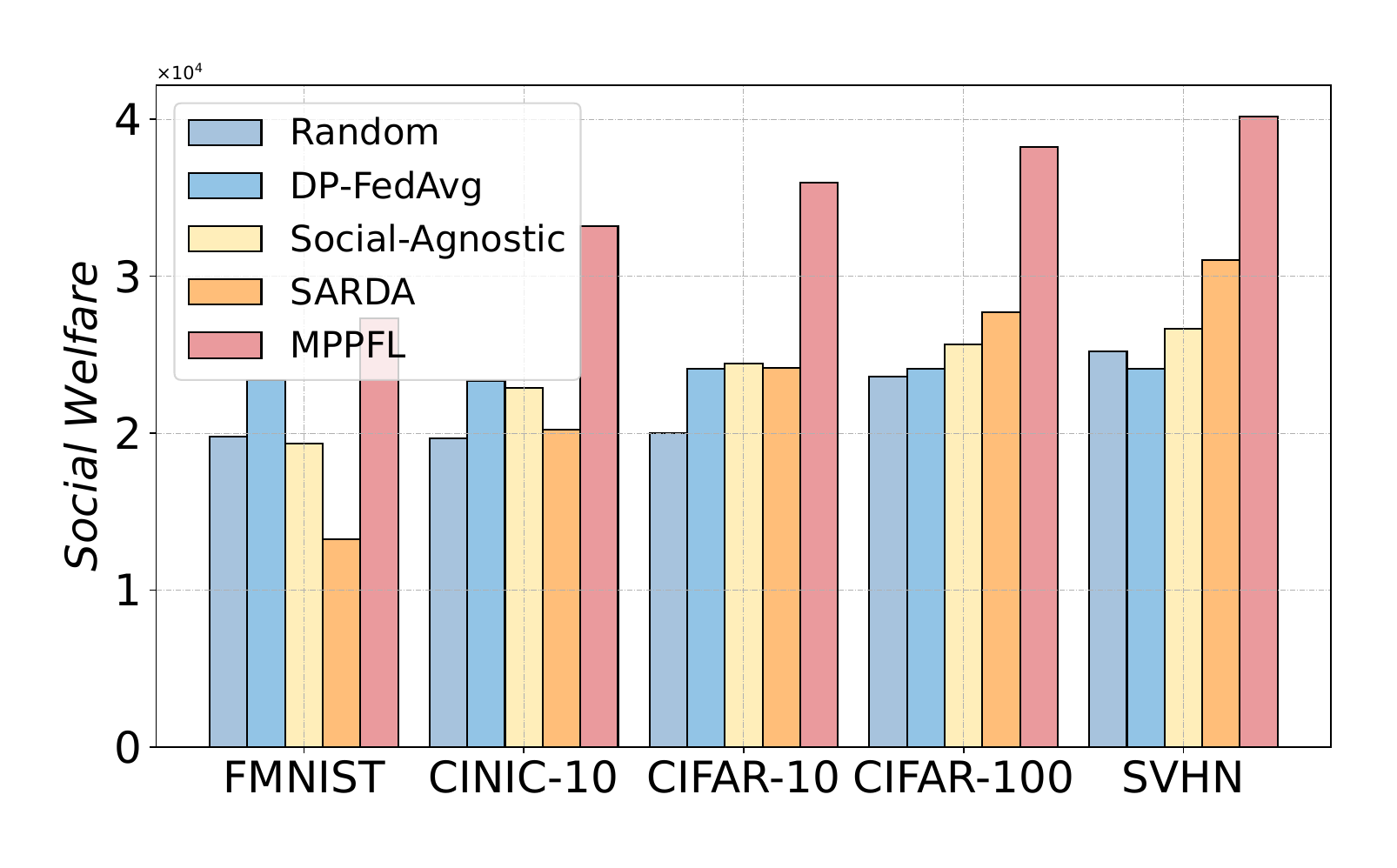}
        \vspace{-20pt}
        \caption{The social welfare comparison.}
        \label{sw_compare}
    \end{minipage}
\vspace{-10pt}
\end{figure*}

Then, we conduct rigorous analysis for the existence of the fixed point for the mean-field estimator $\phi_{i}(t)$, and develop an iterative algorithm to finalize the optimal strategy design.
\begin{theorem}[Fixed Point Existence] \label{theorem_fixed_point}
There exists a fixed point for the mean-field estimator $\phi_{i}(t)$, which can be attained through Alg.~\ref{alg_fix_point}.
\end{theorem}
\begin{proof}[Proof Sketch]
Firstly, $\boldsymbol{\widetilde{W}}^k$ remains row-normalized for positive integer $k$ due to row-normalized matrix $\boldsymbol{\widetilde{W}}$, and the multi-hop leakage coefficient satisfies $\sigma_{ij} = \sum\nolimits_{k=1}^K \lambda^{k-1} (\boldsymbol{\widetilde{W}}^k)_{ij} \leq S := \frac{1 - \lambda^K}{1 - \lambda}$, revealing an exponential decay with the propagation length $K$. Next, by substituting the mean-field estimator in Definition~\ref{mean_field} into Eq.~(\ref{optimal_privacy_budget}), we construct a continuous mapping $\Theta$ regarding the privacy budget $\rho_{i}(t)$, $t \!\in\! \{0, 1, \ldots, T\}$, $i \!\in\! \{1, 2, \ldots, N\}$ on the continuous space $\Omega = [\rho_{l}, \rho_{h}]^{N(T+1)}$, where $\rho_{i}(t) \!\in\! [\rho_{l}, \rho_{h}]$, $\rho_{l} \!=\! \frac{m_l-\alpha S m_h}{1-\alpha^{2} S^{2}}$, $\rho_{h} \!=\! \frac{m_h-\alpha S m_l}{1-\alpha^{2} S^{2}}$, $m_l \!=\! \min_{i,t}\frac{r(t) - b_{i}}{2a_{i}}$, $m_h \!=\! \max_{i,t}\frac{r(t) - b_{i}}{2a_{i}}$ and $\alpha S \!<\! 1$. Thus, $\Theta$ is a continuous self-mapping on the compact convex space from $\Omega$ to $\Omega$ and by Brouwer’s Fixed-Point Theorem, $\Theta$ has a fixed point in $\Omega$. Then, to quantify the divergence between any two variables, we introduce a distance function $\Gamma$ defined by the $\ell_2$-norm and have $\Gamma(\boldsymbol{\rho}, \boldsymbol{\rho^{\prime}}) \leq \nu \| \boldsymbol{\rho} - \boldsymbol{\rho^{\prime}} \|_2^2$, where $\nu = \alpha^2 S^2$ satisfies $0 \leq \nu <1$, indicating that the mapping $\Theta$ is a contraction mapping on $\Omega$ according to Banach’s Fixed Point Theorem, which guarantees the \emph{existence} and \emph{uniqueness} of the fixed point in Alg.~\ref{alg_fix_point}.   \end{proof}

The iterative calculation of the fixed point of the mean-field estimator $\phi_{i}(t)$ in Alg.~\ref{alg_fix_point} effectively degrades the high computational complexity to the linear scale, denoted by $O(NTM)$, where $M$ represents the iterative round required to achieve $\epsilon_{0}$. The complete process of our proposed framework MPPFL is shown in Alg.~\ref{alg2}.

\begin{algorithm}[t]
\small
\caption{FL algorithm with MPPFL}
\label{alg2}
\SetAlgoLined
\KwIn{$N$, $T$, $\boldsymbol{\phi}(t)$, learning rate $\eta$, weight $\theta_{i}$ ,$\tau$, discount factor $\alpha$, where $t \!\in\! \{0, 1, \ldots, T\}$.}
\KwOut{The global model parameter $\boldsymbol{w(T)}$.}
\For{$t \leftarrow 0, 1, \ldots, T$}{
    The central server decides the optimal unit-reward $r^{*}(t)$ based on Eq.~(\ref{optimal_unit_reward}) \;
    Each client decides the optimal privacy budget $\rho_{i}^{*}(t)$ based on Eq.~(\ref{optimal_privacy_budget})\;
    \ForEach{client}{
        Local training: $\ \boldsymbol{w_{i}(t \!+\! 1)} \!=\! \boldsymbol{w(t)} \!-\! \eta \nabla F_{i}(\boldsymbol{w(t)})$ \;
        Gradient Perturbation: $\nabla \widetilde{F}_{i}(\boldsymbol{w(t)}) \!=\! \nabla F_{i}(\boldsymbol{w(t)}) \!+\! \boldsymbol{n_{i}(t)}$ \;
    }
    The central server updates the global model with: $w(t + 1) = w(t) - \eta \sum_{i=1}^{N} \! \theta_{i} \nabla \widetilde{F}_{i}(\boldsymbol{w(t)})$ 
}
\end{algorithm}

\section{Efficiency Analysis of Equilibrium Strategy}\label{s4}
The incorporation of external privacy risks is driven by individual concerns, and such local decisions can have profound implications on the collective coordination of all clients. Specifically, as each client internalizes its external impact by considering $\mathcal{R}_i(t)$ in their optimization objectives, the equilibrium may significantly differ from that obtained by clients merely considering their internal leakage $\rho_i(t)$, which naturally leads to the following question: \emph{How do such behavioral differences affect overall system efficiency?} To address this, we introduce the PoA metric, which quantifies the efficiency degradation incurred by decentralized decision-making relative to the socially optimal outcome.

Firstly, we begin by formulating the social welfare objective under the decentralized learning system, which runs for $T$ iterations and involves $N$ participating clients. Specifically, at $t$-th iteration, client $x_i$, $i \in \{1, \ldots, N\}$ designs the privacy budget $\rho_i(t)$, undertakes privacy costs and receives the corresponding monetary rewards. Thus, the cumulative social welfare across all clients and iterations is accordingly defined as follows:
\begin{align} \label{social_welfare_function}
\mathcal{W}(\boldsymbol{\rho}) = \sum\nolimits_{t=0}^{T} \sum\nolimits_{i=1}^{N} [ r(t)\rho_i(t) - ( a_i s_i^2(t) + b_i s_i(t) ) ].
\end{align}

Let $\boldsymbol{\rho}^{\text{SW}*}$ denote the strategy that maximizes the social welfare objective function $\mathcal{W}$. To quantify the inefficiency arising from self-interested behavior, we next introduce the definition of the Price of Anarchy (PoA)~\citep{papadimitriou2001algorithms}, which compares the worst-case Nash Equilibrium to the optimal social welfare.

\begin{definition}[Price of Anarchy, PoA] \label{definition_poa}
The PoA is defined as the ratio between the optimal social welfare and the minimum social welfare achieved  by the worst-case Stackelberg Nash Equilibrium, i.e.,
\begin{align} 
\text{PoA} = \frac{\mathcal{W}(\boldsymbol{\rho}^{\text{SW}*})}{\min_{\rho^*_i(t), i \in \{1,\ldots,N\}, t \in \{0,\ldots,T\}} \mathcal{W}(\boldsymbol{\rho}^*)},
\end{align}
where the value of PoA closer to 1 indicates that the Nash Equilibrium achieves approximately-optimal social efficiency. In contrast, a larger value of PoA implies a greater loss in efficiency due to uncoordinated, self-interested behaviors across clients.
\end{definition}






\paragraph{PoA Analysis under Social-Agnostic.}
We firstly analyze the efficiency of the Nash Equilibrium under social-agnostic, i.e., clients only consider internal privacy costs without the risk $\mathcal{R}_i(t)$. This scenario reflects a common decentralized behavior, where each client is unaware of or unwilling to internalize the indirect privacy risks caused by others. Denote the Nash Equilibrium strategy when clients solely optimize their internal privacy cost by $\boldsymbol{\rho}^{\text{SA}*}$. Thus, according to Definition~\ref{definition_poa}, the PoA under social-agnostic scenario is characterized by $\text{PoA}^{\text{(SA)}} = \frac{\mathcal{W}(\boldsymbol{\rho}^{\text{SW}*})}{\mathcal{W}(\boldsymbol{\rho}^{\text{SA}*})}$, which is strictly higher than 1. Then, we present the lower boundary of $\text{PoA}^{\text{(SA)}}$ in the following theorem.


\begin{theorem} \label{theorem_poa_SA}
Under any nontrivial influence, i.e., there exists at least one $\sigma_{ij} > 0$, the equilibrium is strictly suboptimal in social welfare and satisfies $\text{PoA}^{\text{(SA)}} > \frac{1}{1-\varepsilon^{2}}$, where $\varepsilon = \frac{\alpha S \tilde{\omega}_{\min}}{(1-\alpha^{2} S^{2})}(\alpha S - \frac{m_l}{m_h})$.
\end{theorem}

\begin{proof}[Proof Sketch]
Under social-agnostic scenarios, we obtain the privacy budget by $\rho_i^{\text{SA}*}(t) = \frac{r(t)-b_i}{2a_i}$, which neglects the privacy risk $\mathcal{R}_i(t)$ from other clients comparing to Eq.~(\ref{optimal_privacy_budget}) in Theorem~\ref{optimal_rho}. By calculating the optimal social welfare based on privacy budget vectors $\boldsymbol{\rho}^{\text{SW}*}$ and $\boldsymbol{\rho}^{\text{SA}*}$, we obtain the lower bound of $\text{PoA}^{\text{(SA)}}$ in Theorem~\ref{theorem_poa_SA}. \end{proof}

By Theorem~\ref{theorem_poa_SA}, the system under social-agnostic scenarios performs worse as client heterogeneity increases, reflected by a decreasing value of $\frac{m_l}{m_h}$, which serves as a proxy for the level of client heterogeneity. As $\frac{m_l}{m_h} \rightarrow 0$, $\text{PoA}^{\text{(SA)}}$ approaches infinity, and the efficiency of the FL system becomes arbitrarily bad. Thus, it is crucial to design an efficient mechanism to enhance social welfare.







\paragraph{PoA Analysis under External Privacy Awareness.} To further decrease clients' privacy loss in the model training, which can be degraded by the decisions of other clients through social networks, we adopt the multi-hop propagation model to quantify the external privacy of each client. We denote the Nash Equilibrium strategy under MPPFL framework by $\boldsymbol{\rho}^{\text{MPP}*}$ and characterize $\text{PoA}^{\text{(MPP)}} = \frac{\mathcal{W}(\boldsymbol{\rho}^{\text{SW}*})}{\mathcal{W}(\boldsymbol{\rho}^{\text{MPP}*})}$. Then, we prove that the multi-hop propagation model achieves convergence in terms of the PoA in the following Theorem.

\begin{theorem} \label{theorem_poa__mppfl}
Our proposed MPPFL framework results in $\text{PoA}^{(\text{MPP})} \\ \rightarrow 1$ with the threshold $\epsilon_{0} \rightarrow 0$.
\end{theorem}

\begin{proof}[Proof Sketch]
Substituting the optimal privacy budget derived in Theorem~\ref{optimal_rho} into the cumulative social welfare function in Eq.~(\ref{social_welfare_function}), and based on Definition~\ref{definition_poa}, we can easily obtain $\text{PoA}^{\text{(MPP)}} = 1 + \frac{\alpha r(t) (\mathcal{R}_i(t) - N \phi_{i}(t))}{\sum\nolimits_{t=0}^{T} \sum\nolimits_{i=1}^{N} [\frac{(r(t)-b_{i})^2}{4a_i} - \alpha r(t) \mathcal{R}_i(t)]}$, where the value of $\text{PoA}^{\text{(MPP)}}$ is solely governed by the difference $\mathcal{R}_i(t) - N \phi_{i}(t)$. Then, according to Definition~\ref{mean_field}, the dependency between $\mathcal{R}_i(t)$ and $N \phi_{i}(t)$ is determined by the threshold $\epsilon_{0}$ in Alg.~\ref{alg_fix_point}. As the value of $\epsilon_{0}$ decrease to 0, $N \phi_{i}(t)$ approaches $\mathcal{R}_i(t)$, leading to $\text{PoA}^{\text{(MPP)}} \rightarrow 1$.
\end{proof}

From Theorem~\ref{theorem_poa__mppfl}, our proposed MPPFL framework incorporates both clients' own privacy and external privacy costs, which decreases the efficiency gap between Nash Equilibrium and social optimum, driving the decentralized Nash Equilibrium closer to the social optimum and PoA decreasing to 1. In contrast, under social-agnostic scenarios, clients’ selfish strategies diverge further from the social optimum, and PoA remains comparatively high. 

\begin{table}[t]
\renewcommand\arraystretch{1.0}
\caption{Accuracy Comparison on Different Datasets.}
\vspace{-25pt}
\begin{center}
\resizebox{0.49\textwidth}{!}{\LARGE \begin{tabular}
{cccccc} 
\toprule[1.5pt]
\textbf{Datasets \#}& \textbf{Data Partition} & \textbf{SA (w/o EPR)} & \textbf{DP-FedAvg}\citep{mcmahan2018learning} & \textbf{SARDA}\citep{sun2024socially} & \textbf{MPPFL}(Ours) \\ \cmidrule[1.2pt](l{1pt}r{0pt}){1-6}

\multirow{3}{*}{\multirowcell{1}{\textbf{FMNIST}}}
        & IID & 89.12 & 88.89 ($\downarrow$ 0.23) & 83.77 ($\downarrow$ 5.35) & 89.05 (\textbf{$\downarrow$ 0.07})  \\ 

        & Dir = 0.3 & 87.91 & 87.05 ($\downarrow$ 0.86) & 77.15 ($\downarrow$ 10.76) & 87.78 (\textbf{$\downarrow$ 0.13})  \\ 
        
        & Dir = 0.6 & 88.25 & 87.69 ($\downarrow$ 0.56) & 80.89 ($\downarrow$ 7.36) & 88.09 (\textbf{$\downarrow$ 0.16})  \\ \midrule[1.2pt]

\multirow{3}{*}{\multirowcell{1}{\textbf{CINIC-10}}}
        & IID & 48.74 & 48.36 (\textbf{$\downarrow$ 0.38})  & 46.13 ($\downarrow$ 2.61) & 48.23 ($\downarrow$ 0.51)  \\ 

        & Dir = 0.3 & 44.68 & 43.66 ($\downarrow$ 1.02) & 40.81 ($\downarrow$ 3.87) & 43.99 (\textbf{$\downarrow$ 0.69})  \\ 
        
        & Dir = 0.6 & 47.36 & 46.26 ($\downarrow$ 1.1) & 43.91 ($\downarrow$ 3.45) & 46.74 (\textbf{$\downarrow$ 0.62})  \\ \midrule[1.2pt]

\multirow{3}{*}{\multirowcell{1}{\textbf{CIFAR-10}}}
        & IID & 66.49 & 65.35 ($\downarrow$ 1.14) & 53.54 ($\downarrow$ 12.95) & 66.29 (\textbf{$\downarrow$ 0.2})  \\ 

        & Dir = 0.3 & 62.44 & 59.84 ($\downarrow$ 2.6) & 41.19 ($\downarrow$ 21.25) & 62.42 (\textbf{$\downarrow$ 0.02})  \\ 
        
        & Dir = 0.6 & 64.47 & 61.37 ($\downarrow$ 3.1) & 46.37 ($\downarrow$ 18.1) & 64.07 (\textbf{$\downarrow$ 0.4})  \\  \midrule[1.2pt]

\multirow{3}{*}{\multirowcell{1}{\textbf{CIFAR-100}}}
        & IID & 48.83 & 46.52 ($\downarrow$ 2.31) & 34.86 ($\downarrow$ 13.97) & 47.83 (\textbf{$\downarrow$ 1.0})  \\ 

        & Dir = 0.3 & 46.53 & 44.40 ($\downarrow$ 2.13) & 32.67 ($\downarrow$ 13.86) & 46.40 (\textbf{$\downarrow$ 0.13})  \\ 
        
        & Dir = 0.6 & 47.98 & 44.88 ($\downarrow$ 3.1) & 33.59 ($\downarrow$ 14.39) & 47.38 (\textbf{$\downarrow$ 0.6})  \\  \midrule[0.8pt]

\multirow{3}{*}{\multirowcell{1}{\textbf{SVHN}}}
        & IID & 87.55 & 87.43 ($\downarrow$ 0.12) & 83.57 ($\downarrow$ 3.98) & 87.51 (\textbf{$\downarrow$ 0.04}) \\ 

        & Dir = 0.3 & 85.89 & 85.78 (\textbf{$\downarrow$ 0.11}) & 81.91 ($\downarrow$ 3.98) & 85.59 ($\downarrow$ 0.3) \\ 
        
        & Dir = 0.6 & 87.23 & 85.89 ($\downarrow$ 1.34) & 83.71 ($\downarrow$ 3.52) & 86.07 (\textbf{$\downarrow$ 1.16})  \\  
         
\bottomrule[1.5pt]
\end{tabular}}
\vspace{-5pt}
\label{accuracy_compare}
\end{center}
\end{table}

\begin{table*}[t]
\renewcommand\arraystretch{1.1}
\caption{Accuracy Comparison on Different Datasets with Different N.}
\vspace{-25pt}
\begin{center}
{\setlength{\tabcolsep}{2pt}
\resizebox{1.0\textwidth}{!}{\begin{tabular}
{ccccc|ccc|ccc|ccc} 
\toprule[1.2pt]

\multirow{2}{*}{\multirowcell{2}{\centering\textbf{Datasets}}} & \multirow{2}{*}{\multirowcell{2}{\textbf{Data} \\ \textbf{Partition}}} & \multicolumn{3}{c}{\centering\textbf{SA (w/o EPR)}} & \multicolumn{3}{c}{\centering\textbf{DP-FedAvg}\citep{mcmahan2017communication}} & \multicolumn{3}{c}{\centering\textbf{SARDA}\citep{sun2024socially}} & \multicolumn{2}{c}{\centering\textbf{MPPFL}(Ours)}   \\ \cmidrule[0.5pt](l{1pt}r{0pt}){3-14}

&  & $N=20$ & $N=50$ & $N=80$ & $N=20$ & $N=50$ & $N=80$ & $N=20$ & $N=50$ & $N=80$ & $N=20$ & $N=50$ & $N=80$  \\ \cmidrule[0.8pt](l{1pt}r{0pt}){1-14}

\multirow{3}{*}{\multirowcell{1}{\textbf{CIFAR-10}}}
        & IID & 66.49 & 67.48 & 67.67 & 65.35 ($\downarrow$ 1.14)  & 65.69 ($\downarrow$ 1.79) & 66.05 ($\downarrow$ 1.62) & 53.54 ($\downarrow$ 12.95) & 35.12 ($\downarrow$ 29.29) & 40.86 ($\downarrow$ 25.07 & 65.96 ($\downarrow$ \textbf{0.2}) & 66.99 ($\downarrow$ \textbf{0.49}) & 67.39 ($\downarrow$ \textbf{0.28})  \\ 

        & Dir = 0.3 & 62.44 & 64.41 & 65.93 & 59.84 ($\downarrow$ 2.6)  & 62.92 ($\downarrow$ 1.49) & 64.01 ($\downarrow$ 1.92) & 41.19 ($\downarrow$ 21.25) & 21.62 ($\downarrow$ 42.79) & 27.39 ($\downarrow$ 38.54) & 62.42 ($\downarrow$ \textbf{0.02}) & 63.89 ($\downarrow$ \textbf{0.52}) & 65.64 ($\downarrow$ \textbf{0.29})  \\ 
        
        & Dir = 0.6 & 64.47 & 65.75 & 66.77 & 61.37 ($\downarrow$ 3.1)  & 63.88 ($\downarrow$ 1.87) & 65.56 ($\downarrow$ 1.21) & 46.37 ($\downarrow$ 18.1) & 30.43 ($\downarrow$ 35.32) & 33.83 ($\downarrow$ 32.94) & 64.07 ($\downarrow$ \textbf{0.4}) & 64.59 ($\downarrow$ \textbf{1.16}) & 66.04 ($\downarrow$ \textbf{0.73})  \\  \midrule[0.5pt]

\multirow{3}{*}{\multirowcell{1}{\textbf{CINIC-10}}}
        & IID & 48.74 & 49.93 & 50.37 & 48.36 ($\downarrow$ \textbf{0.38})  & 49.01 ($\downarrow$ 0.92) & 49.68 ($\downarrow$ 0.69) & 46.13 ($\downarrow$ 2.61) & 43.59 ($\downarrow$ 6.34) & 46.08 ($\downarrow$ 4.29) & 48.23 ($\downarrow$ 0.51) & 49.37 ($\downarrow$ \textbf{0.01}) & 49.28 ($\downarrow$ \textbf{0.41})  \\ 

        & Dir = 0.3 & 44.68 & 47.20 & 48.01 & 43.66 ($\downarrow$ 1.03)  & 46.41 ($\downarrow$ 0.79) & 47.55 ($\downarrow$ 1.36) & 40.81 ($\downarrow$ 3.87) & 41.27 ($\downarrow$ 5.93) & 41.11 ($\downarrow$ 6.9) & 43.99 ($\downarrow$ \textbf{0.69}) & 46.89 ($\downarrow$ \textbf{0.31}) & 47.32 ($\downarrow$ \textbf{0.69})  \\ 
        
        & Dir = 0.6 & 47.36 & 49.38 & 49.69 & 46.26 ($\downarrow$ 1.1)  & 48.02 ($\downarrow$ 1.36) & 49.54 ($\downarrow$ 0.15) & 43.91 ($\downarrow$ 3.45) & 42.10 ($\downarrow$ 7.28) & 42.64 ($\downarrow$ 7.05) & 46.74 ($\downarrow$ \textbf{0.62}) & 48.97 ($\downarrow$ \textbf{0.41}) & 49.66 ($\downarrow$ \textbf{0.03})  \\  
         
\bottomrule[1.2pt]
\end{tabular}}}
\vspace{-6pt}
\label{accuracy_N_compare}
\end{center}
\end{table*}

\begin{figure*}[t]
\setlength{\abovecaptionskip}{2pt}
\centering
    \begin{minipage}{122pt}
        \includegraphics[width=1.0\textwidth, trim=5 5 10 5,clip]{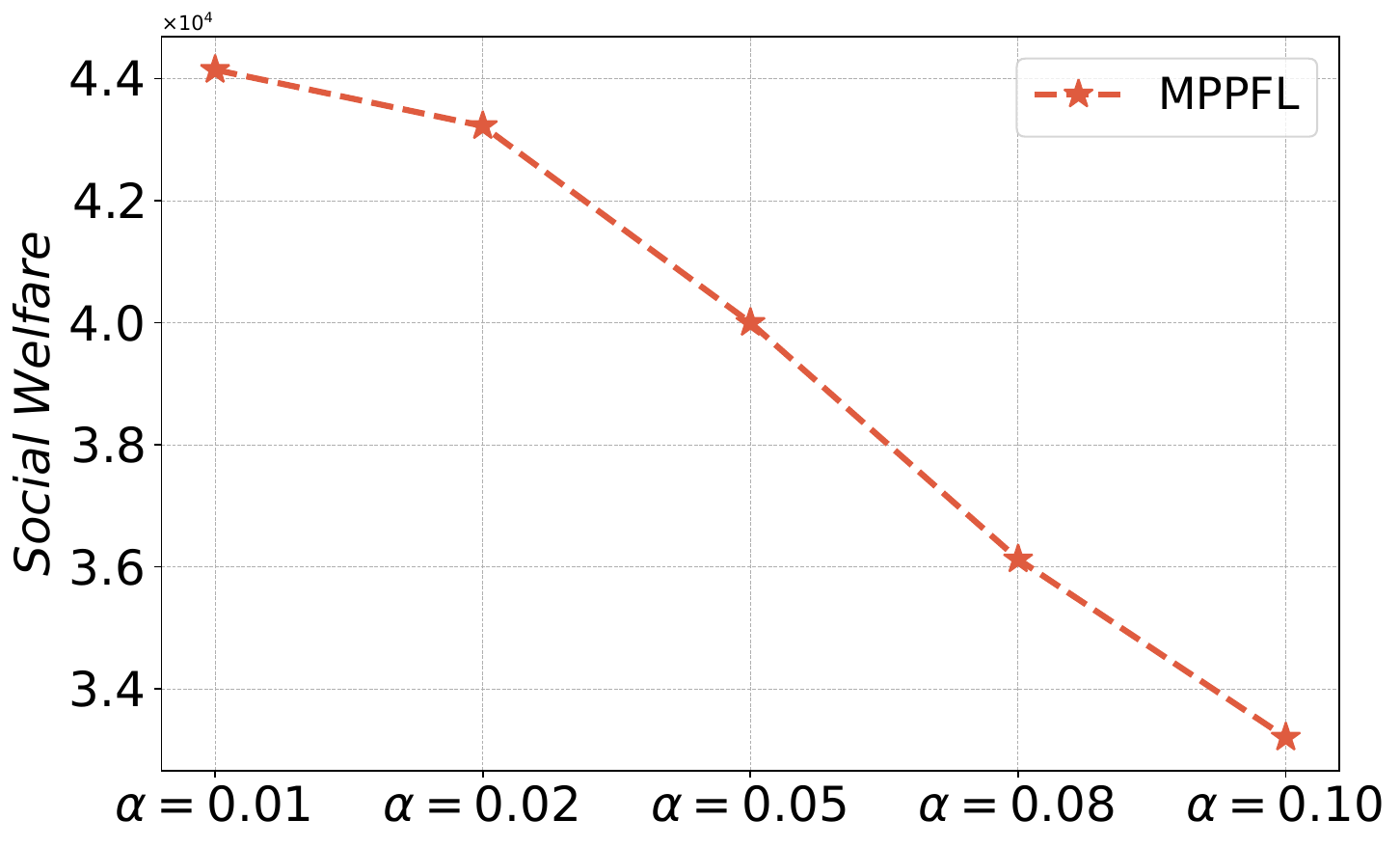}
        \vspace{-20pt}
        \caption{The social welfare comparison on CINIC-10 dataset with different $\alpha$.}
        \label{sw_compare_alpha}
    \end{minipage}
    \begin{minipage}{122pt}
        \includegraphics[width=1.0\textwidth, trim=5 10 8 5,clip]{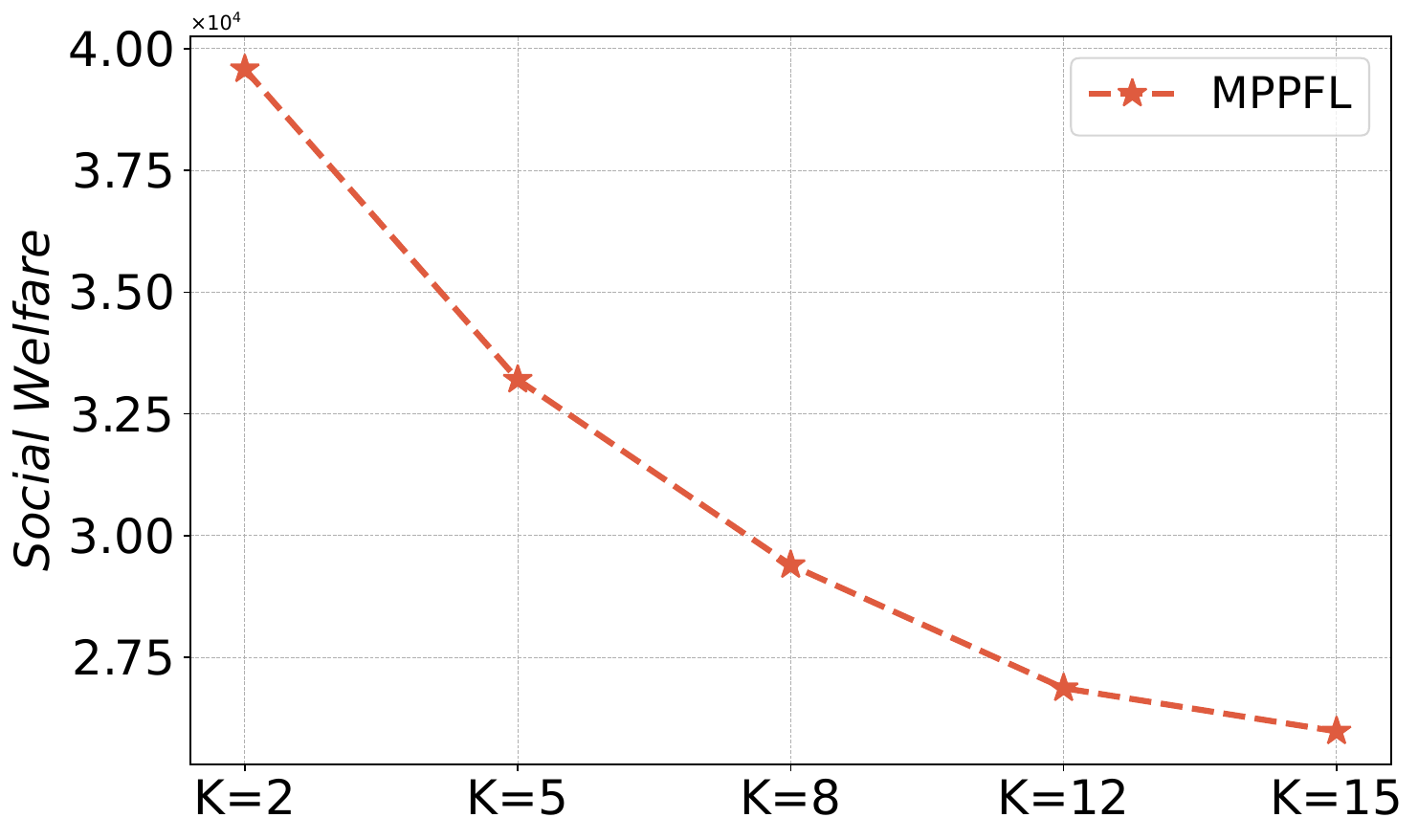}
        \vspace{-20pt}
        \caption{The social welfare comparison on CINIC-10 dataset with different $K$.}
        \label{sw_k}
    \end{minipage}
    \hspace{2pt}
    \begin{minipage}{122pt}
        \includegraphics[width=1.0\textwidth, trim=35 40 30 35,clip]{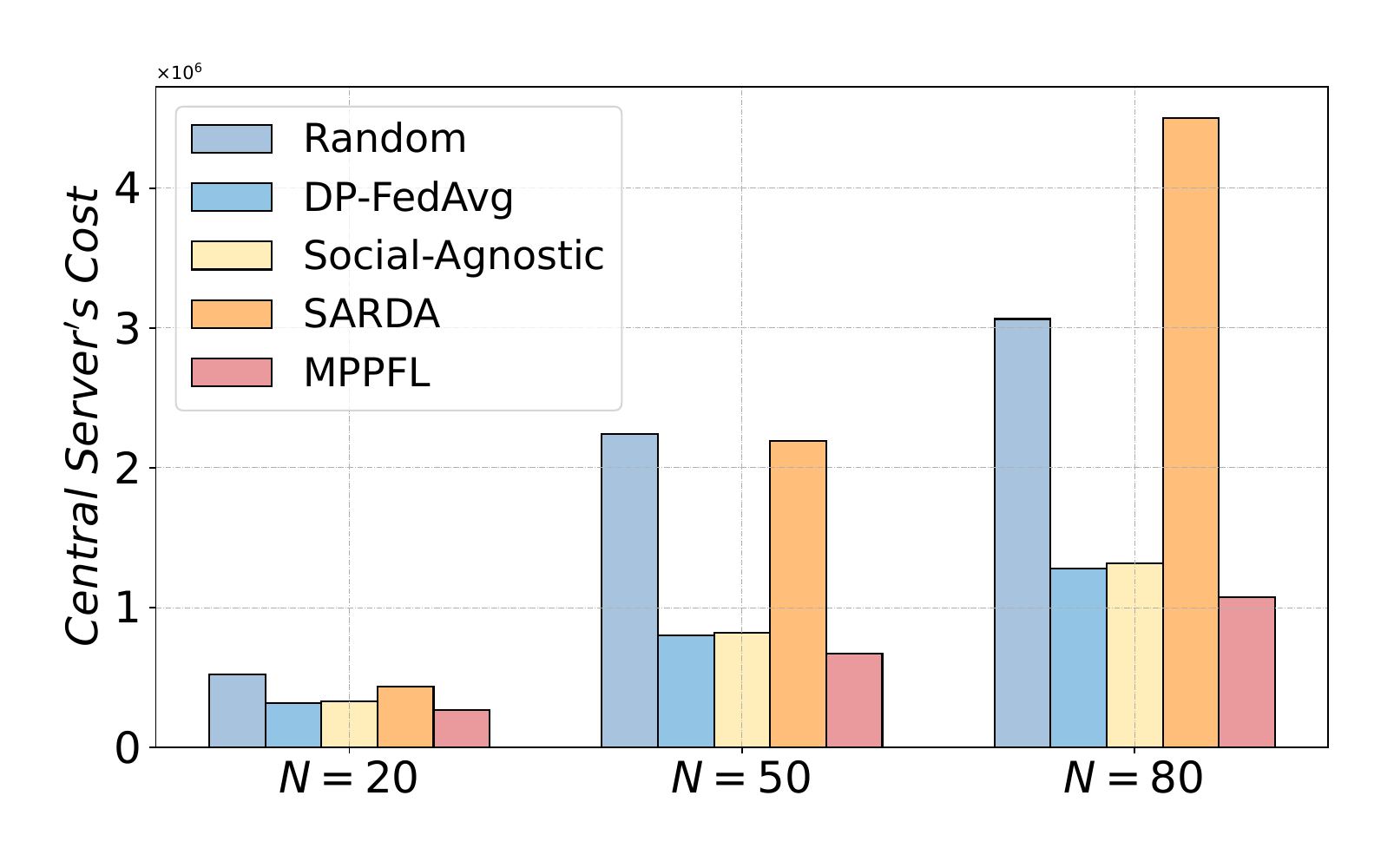}
        \vspace{-20pt}
        \caption{The central server’s cost comparison on CINIC-10 dataset with different $N$.}
        \label{csc_cinic}
    \end{minipage}
    \hspace{2pt}
    \begin{minipage}{122pt}
        \includegraphics[width=1.0\textwidth, trim=35 40 30 35,clip]{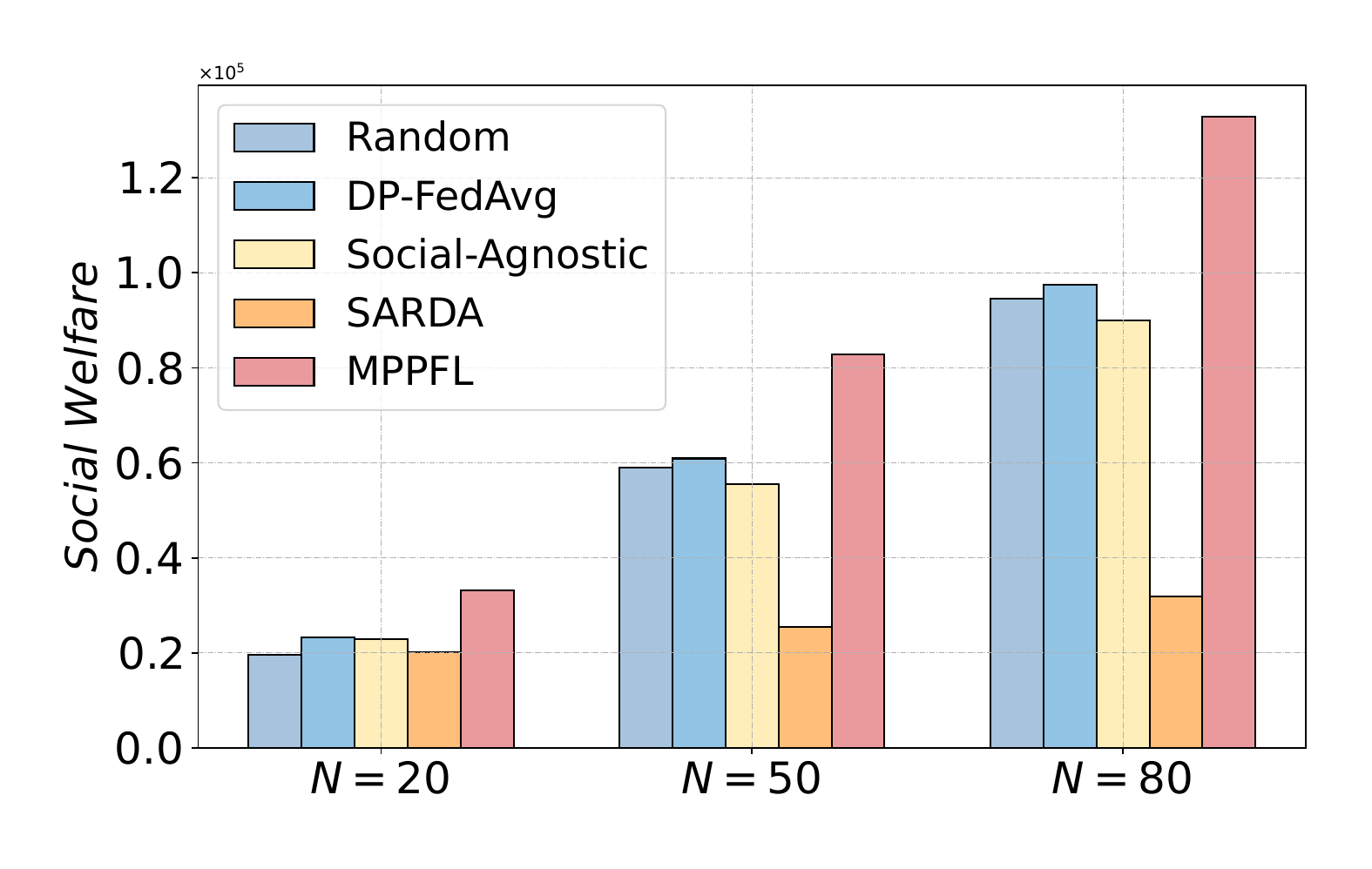}
        \vspace{-20pt}
        \caption{The social welfare comparison on CINIC-10 dataset with different $N$.}
        \label{sw_cinic}
    \end{minipage}
\vspace{-6pt}    
\end{figure*}

\section{Experiments}

\subsection{Experimental Setups}
\paragraph{Datasets and Hyperparameter settings.}
We conduct our experiments on five different real-world datasets, i.e., FMNIST~\citep{xiao2017fashion}, CINIC-10~\citep{darlow2018cinic}, CIFAR-10 \& CIFAR-100~\citep{krizhevsky2009learning}, and SVHN~\citep{netzer2011reading}, to validate the performance of our proposed MPPFL. We randomly distribute the training data to each participating client and implement different CNN models similar to the model architectures in~\citep{mcmahan2017communication} to train the datasets FMNIST and CIFAR-10 and utilize the VGG-11 and VGG-16 models for the SVHN dataset and CIFAR-100. We apply the SGD optimizer for local model training and default the learning rate $\eta = 0.01$, local training epoch $L = 5$ for all datasets. Specifically, we set the batchsize $B = 32$ for FMNIST and CIFAR-10 and $B = 256$ for SVHN. Regarding the parameters for our MPPFL system, we default the longest privacy propagation length $K=5$, the client number $N \in \{20, 50, 80\}$, and the threshold $\epsilon_{0} = 10^{\text{-}3}$ for obtaining the fixed point $\phi_i(t)$ in Alg.~\ref{alg_fix_point}. $\alpha \!\in\! [0.01, 0.1]$ to balance the magnitude difference between $\rho_i(t)$ and $\mathcal{R}_{i}(t)$ of $s_i(t)$ in Eq. (\ref{client_utility}), ensuring the effectiveness of the optimization problem. Dirichlet partitioning is employed to simulate non-independent and identically distributed (non-IID) data across clients, where local datasets are constructed by sampling label distributions from a Dirichlet distribution with concentration parameters~$\{0.3, 0.6\}$.

\paragraph{Baselines.}
To evaluate the effectiveness of our proposed MPPFL framework, we compare it with four baselines. Random strategy assigns each client’s privacy budget uniformly at random within a fixed range. DP-FedAvg \citep{mcmahan2018learning} adopts a fixed privacy budget across all clients and iterations, ignoring individual heterogeneity and social interactions. SA strategy assumes that each client optimizes its privacy budget solely based on its utility trade-off, without considering the external privacy risks (w/o EPR). SARDA \citep{sun2024socially} is a socially-aware iterative double-auction mechanism that incorporates inter-client privacy dependencies arising from data correlations in social networks.

\paragraph{ER model-based Social network.}
To model a heterogeneous probabilistic social network, we construct a weighted directed social network with $N$ clients based on the non-uniform extension of the Erdős–Rényi model, following \citep{lin2021friend}. Specifically, for each client pair $(i, j)$, the connection is formed with probability $p_{ij} \sim \mathbb{U}(0.1, 0.9)$. The edge weight $\omega_{ij}$, representing the strength of social relation, is sampled from a uniform distribution $\mathbb{U}(0.1, 1.0)$.

\subsection{Experimental results}
\paragraph{Utility Evaluation Analysis.}
Figure \ref{phi_all_t_convergence} shows that $\phi_1(t)$ converges rapidly within the first $15$ rounds under different global iterations, confirming the convergence and linear computation complexity of Alg.~\ref{alg_fix_point}. As illustrated in Figure \ref{scs_compare}, leveraging the optimal unit-reward $r^*(t)$ derived in Lemma \ref{optimal_unit_reward_lemma}, the MPPFL method consistently yields the lowest central server cost across all datasets, demonstrating its effectiveness in minimizing the central server's cost. To validate the theoretical derivations in Section \ref{s4}, Figure \ref{sw_compare} shows that MPPFL attains the highest social welfare among all baselines, with the performance gap widening on more complex datasets, which indicates the robustness of the MPPFL across diverse learning scenarios. 
Figure~\ref{sw_compare_alpha} and Figure~\ref{sw_k} illustrate the impact of external privacy risk sensitivity $\alpha$ and propagation length $K$ on social welfare. As shown in Figure~\ref{sw_compare_alpha}, increasing $\alpha$, which quantifies clients’ sensitivity to external privacy risks, leads to a gradual decline in social welfare under MPPFL, as clients become more conservative in their privacy budget decisions due to greater emphasis on external privacy risks from others. Similarly, Figure~\ref{sw_k} reveals that the social welfare of MPPFL decreases with longer propagation lengths $K$, due to the accumulation of multi-hop privacy leakage along extended social paths.
To assess the scalability of our framework, we conduct utility comparisons on the CINIC-10 dataset under varying numbers of clients. Figure~\ref{csc_cinic} shows that the central server cost grows with the number of clients across all strategies, while MPPFL consistently incurs the lowest cost regardless of the network size. Correspondingly, Figure~\ref{sw_cinic} indicates that MPPFL achieves the highest social welfare across all settings, with its performance gap widening as $N$ increases, highlighting the scalability of MPPFL in controlling the central server's cost and coordinating client behavior efficiency as the social network expands.

\paragraph{Model Training Performance Analysis.}
We take Social-Agnostic (SA) as the reference baseline, which represents clients' optimal decisions without accounting for external privacy risks through social ties, resulting in less noise being added. Although SA achieves relatively high model accuracy due to weaker privacy preservation, it incurs higher central server costs and yields lower social welfare, as shown in Figs.~\ref{scs_compare}-\ref{sw_compare}. As shown in Table~\ref{accuracy_compare}, MPPFL consistently achieves the closest accuracy to SA across datasets, with much smaller model performance degradation than other baselines under non-IID settings. Compared to SARDA, which also accounts for external privacy risks, MPPFL enforces stronger socially-aware privacy protection while maintaining competitive accuracy and lower central server costs. Table~\ref{accuracy_N_compare} further evaluates the training performance under varying numbers of clients. The results show that MPPFL maintains stable performance as the social network scales, with only marginal accuracy fluctuations, while both DP-FedAvg and SARDA experience larger instability or accuracy degradation. These observations highlight MPPFL’s superior ability to balance privacy preservation and model training performance under different social network scales. The additional experiment results are placed in Appendix E.

\section{Conclusion}
In this paper, we propose a novel privacy-preserving FL framework MPPFL aimed at achieving trade-offs between local privacy preservation and global training objectives while considering multi-hop privacy propagation over social networks. By modeling the server-client interaction as a Stackelberg game and incorporating a mean-field estimator to quantify external privacy risks, we prove the existence and convergence of the fixed point of the mean-field estimator and derive closed-form expressions of the optimal strategy profiles that constitute the Stackelberg Nash Equilibrium. Extensive experiments demonstrate that MPPFL consistently outperforms existing baselines by achieving lower central server costs and higher social welfare, while maintaining strong model accuracy across varying data heterogeneity and client scales.

\begin{ack}
This work was supported in part by the National Natural Science Foundation of China (NSFC) under Grant 62206320 and in part by Guangdong Basic and Applied Basic Research Foundation under Grant 2024A1515010118.
\end{ack}

\newpage
\bibliography{mybibfile}

\newpage

\setcounter{equation}{12}

\appendix
\section{Proof of Proposition \ref{proposition_varience}}
\begin{proof}
Consider client $x_i$ with two adjacent datasets $\mathcal{D}_i$ and $\mathcal{D}_i'$ that differ by a single data point. Assume that the local model update $\boldsymbol{w}_i(t)$ at the $t$-th global iteration is subject to $\ell_2$-norm clipping with threshold $\mathcal{S}$, i.e., $\|\boldsymbol{w}_i(t)\|_2 \leq \mathcal{S}$, to ensure bounded sensitivity. Thus, the $\ell_2$-sensitivity of the query function $Q$, which maps the local dataset to the average model update, is upper bounded by:
\begin{align} \label{sensitivity}
\Delta Q 
&= \max_{\mathcal{D}_i, \mathcal{D}_i^{\prime}} \| Q(\mathcal{D}_i) - Q(\mathcal{D}_i') \|_2 \nonumber \\
&= \max_{\mathcal{D}_i, \mathcal{D}_i^{\prime}} 
\| \frac{1}{|\mathcal{D}_i|} \sum\nolimits_{j \in \mathcal{D}_i} \!\!\! \nabla f_i(\boldsymbol{w}, j) 
\!-\! \frac{1}{|\mathcal{D}_i^{\prime}|} \sum\nolimits_{j \in \mathcal{D}_i^{\prime}} \!\!\! \nabla f_i(\boldsymbol{w}, j) \|_2 \nonumber \\
&\leq \frac{1}{|\mathcal{D}_i|} \| \nabla f_i(\boldsymbol{w}, a) - \nabla f_i(\boldsymbol{w}, b) \|_2 \leq \frac{2\mathcal{S}}{|\mathcal{D}_i|}.
\end{align}

Under the Gaussian mechanism in the $\rho$-$z$CDP mechanism, the noise variance required to achieve a privacy budget of $\rho_i(t)$ is given by $\delta_i^2(t) = \frac{\Delta Q^2}{2\rho_i(t)}$. Applying the sensitivity bound in Eq.~\eqref{sensitivity}, the proposition holds.\end{proof}

\section{Proof of Proposition \ref{global_accuracy_loss}}
\begin{proof}
Following the result in~\citep{rakhlin2011making}, suppose the global loss function $F(\boldsymbol{w}(t))$ satisfies the $\beta$-Lipschitz smoothness and the Polyak–Łojasiewicz condition, the following inequality holds:
\begin{align}
\mathbb{E}[F(\boldsymbol{w}(t)) - F(\boldsymbol{w}^*)] \leq \frac{\beta}{2\mu^2 t} \mathbb{E}[ \| \nabla \widetilde{F}(\boldsymbol{w}(t)) \|_2^2 ],
\label{eq:pl_loss_bound_en} 
\end{align}
where $\mu$ represents the PL constant and the perturbed global gradient $\nabla \widetilde{F}(\boldsymbol{w}(t)) \!=\! \sum\nolimits_{i=1}^N \theta_i (\nabla F_i(\boldsymbol{w}(t)) \!+\! \boldsymbol{n}_i(t) ) \!=\! \nabla F(\boldsymbol{w}(t)) \!+\! \sum\nolimits_{i=1}^N \theta_i \boldsymbol{n}_i(t)$, with the additive noise $\boldsymbol{n}_i(t)$ injected by client $x_i$ and modeled as a zero-mean Gaussian variable with covariance $\delta_i^2(t) \boldsymbol{I}_p$. Since the noise terms are independent and have zero mean, the expectation of the squared norm of the perturbed gradient becomes:
\begin{align}
\mathbb{E}[ \| \nabla \widetilde{F}(\boldsymbol{w}(t)) \|_2^2 ]
&= \mathbb{E}[ \| \nabla F(\boldsymbol{w}(t)) \|_2^2 ] 
+ \mathbb{E}[ \| \sum\nolimits_{i=1}^N \theta_i \boldsymbol{n}_i(t) \|_2^2 ] \nonumber \\ 
&\leq \mathcal{E}^2 + p \sum\nolimits_{i=1}^N \theta_i^2 \delta_i^2(t),
\label{eq:noise_grad_bound_en}
\end{align}
where $\mathbb{E}[ \| \nabla F(\boldsymbol{w}(t)) \|_2^2 ] \leq \mathcal{E}^2$, and $p$ denotes the dimensionality of the model parameter $\boldsymbol{w}$.
Substituting Eq.~(\ref{eq:noise_grad_bound_en}) into Eq.~(\ref{eq:pl_loss_bound_en}) yields the upper bound of the global accuracy loss as expressed in Eq.~(\ref{accuracy_loss}).   \end{proof}

\section{Proof of Lemma \ref{optimal_unit_reward_lemma}}
\begin{proof}
By substituting each client's optimal privacy budget $\rho_{i}^{*}(t)$, as defined by Eq.~(\ref{optimal_privacy_budget}), into the central server's cost function in Eq.~(\ref{central_server_cost}), we obtain the following reformulated expression:
\begin{align}\label{final_central_server_cost}
\!\! U_{t}(r(t), \boldsymbol\rho^{*}(t)) \!=&\ \tau \sum\nolimits_{i=1}^{N} \frac{\epsilon_{i}}{t} (\frac{r(t) - b_{i}}{2a_{i}} - \alpha N \phi_{i}(t))^{-1} \!+\! (1 \!-\! \tau) \nonumber \\ 
&\times  \sum\nolimits_{i=1}^{N} [\frac{r^{2}(t) \!-\! b_{i}r(t)}{2a_{i}} \!-\! \alpha N \phi_{i}(t)r(t)].
\end{align}

We observe that the reformulated cost function in Eq.~(\ref{final_central_server_cost}) is a high-order, non-polynomial function with respect to the unit-reward parameter $r(t)$. Consequently, to analyze the optimality and convexity properties of the cost function, we derive its first-order and second-order derivatives as follows:
\begin{align} 
\!\!\!\!\!\!\!\! \frac{\partial U_{t}}{\partial r(t)} \!=& - \tau \sum\nolimits_{i=1}^{N} \frac{\epsilon_{i}}{2 t a_{i}} ( \frac{r(t)-b_{i}}{2a_{i}} - \alpha N \phi_{i}(t) )^{-2}) \nonumber \\ 
&+ (1 - \tau) \sum\nolimits_{i=1}^{N} [ \frac{2 r(t) - b_{i}}{2 a_{i}}- \alpha N \phi_{i}(t) ],  \label{first_order_CS} \\
\!\!\!\!\!\!\!\! \frac{\partial^{2} U_{t}}{\partial r^{2}(t)} \!=&\ \tau\! \sum\nolimits_{i=1}^{N}\! \frac{\epsilon_{i}}{2 t a_{i}^{2}} ( \frac{r(t) \!-\! b_{i}}{2a_{i}} \!-\! \alpha  N  \phi_{i}(t) )^{-3} \!+\!  \frac{(1 \!-\! \tau)N}{a_{i}}. \!\!\!\!  \label{second_order_CS}
\end{align}

From Eq.~(\ref{second_order_CS}), we hold $\frac{\partial^{2} U_{t}}{\partial r^{2}(t)} > 0$, which implies that the cost function $U_{t}$ is strictly convex in the feasible domain. To demonstrate the existence of the minimum for the cost function in Eq.~(\ref{final_central_server_cost}), we analyze the root of the first-order derivative given in Eq.~(\ref{first_order_CS}) within the feasible region of $r(t) \in (0,+\infty)$:
\begin{align}\label{limit}
\lim_{r(t) \rightarrow 0} \frac{\partial U_{t}}{ \partial r(t)} \rightarrow - \infty, \ \lim_{r(t) \rightarrow +\infty} \frac{\partial U_{t}}{ \partial r(t)} \rightarrow + \infty.
\end{align}

Eq.~(\ref{limit}) shows that the derivative transitions from negative to positive over the domain of $r(t)$, which confirms the existence of at least one root for $\frac{\partial U_t}{\partial r(t)} = 0$. By calculating the equation $\frac{\partial U_{t}}{\partial r(t)} = 0$, we can obtain the function of optimal unit-reward $r^{*}(t)$ in Eq.~(\ref{optimal_unit_reward}).  \end{proof}

\begin{figure*}[t]
\setlength{\abovecaptionskip}{2pt}
\centering
    \begin{minipage}{122pt}
        \includegraphics[width=1.0\textwidth, trim=5 10 8 5,clip]{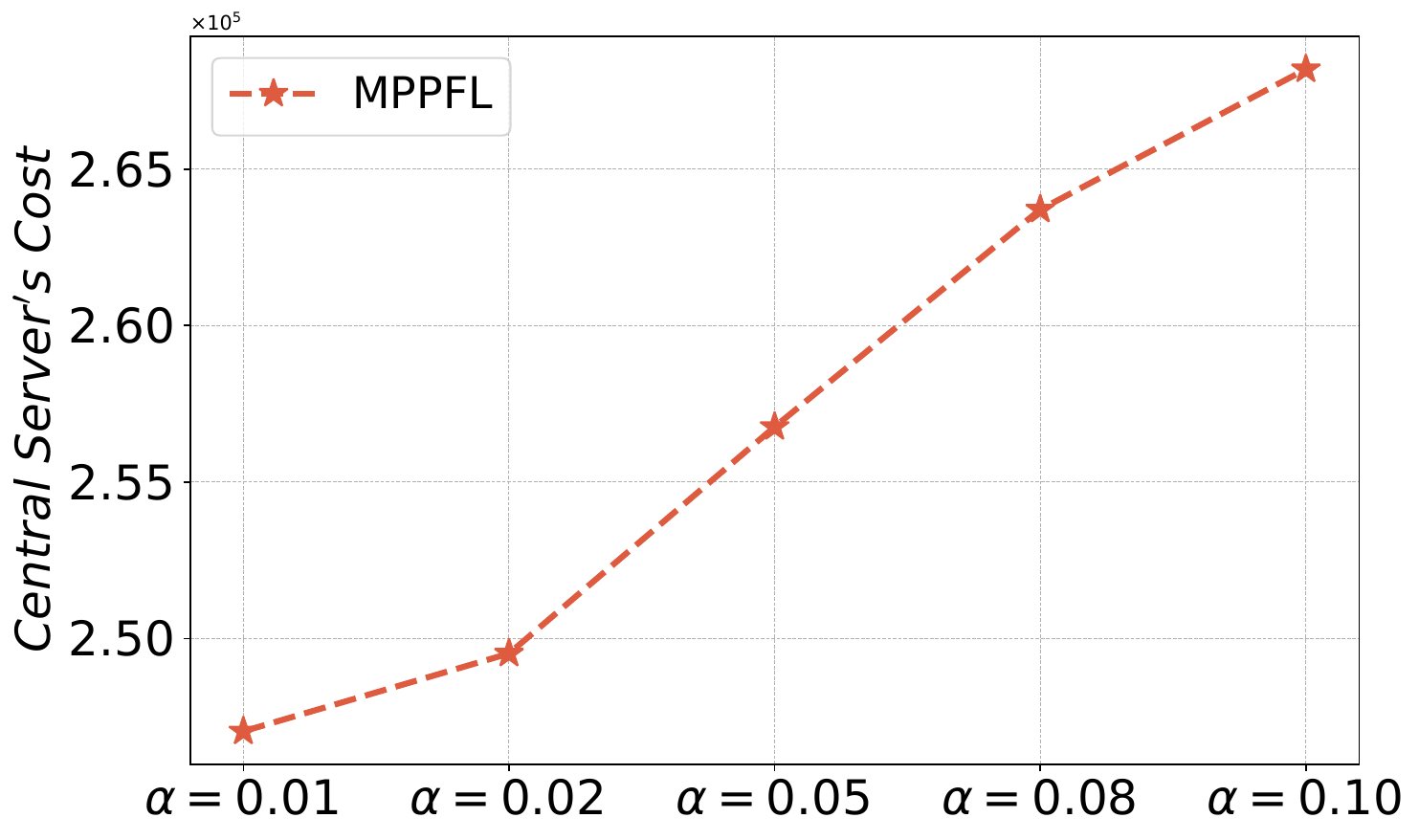}
        \vspace{-20pt}
        \caption{The central server’s cost on CINIC-10 dataset with different $\alpha$.}
        \label{csc_t80_n20_alpha}
    \end{minipage}
    \begin{minipage}{122pt}
        \includegraphics[width=1.0\textwidth, trim=5 10 8 5,clip]{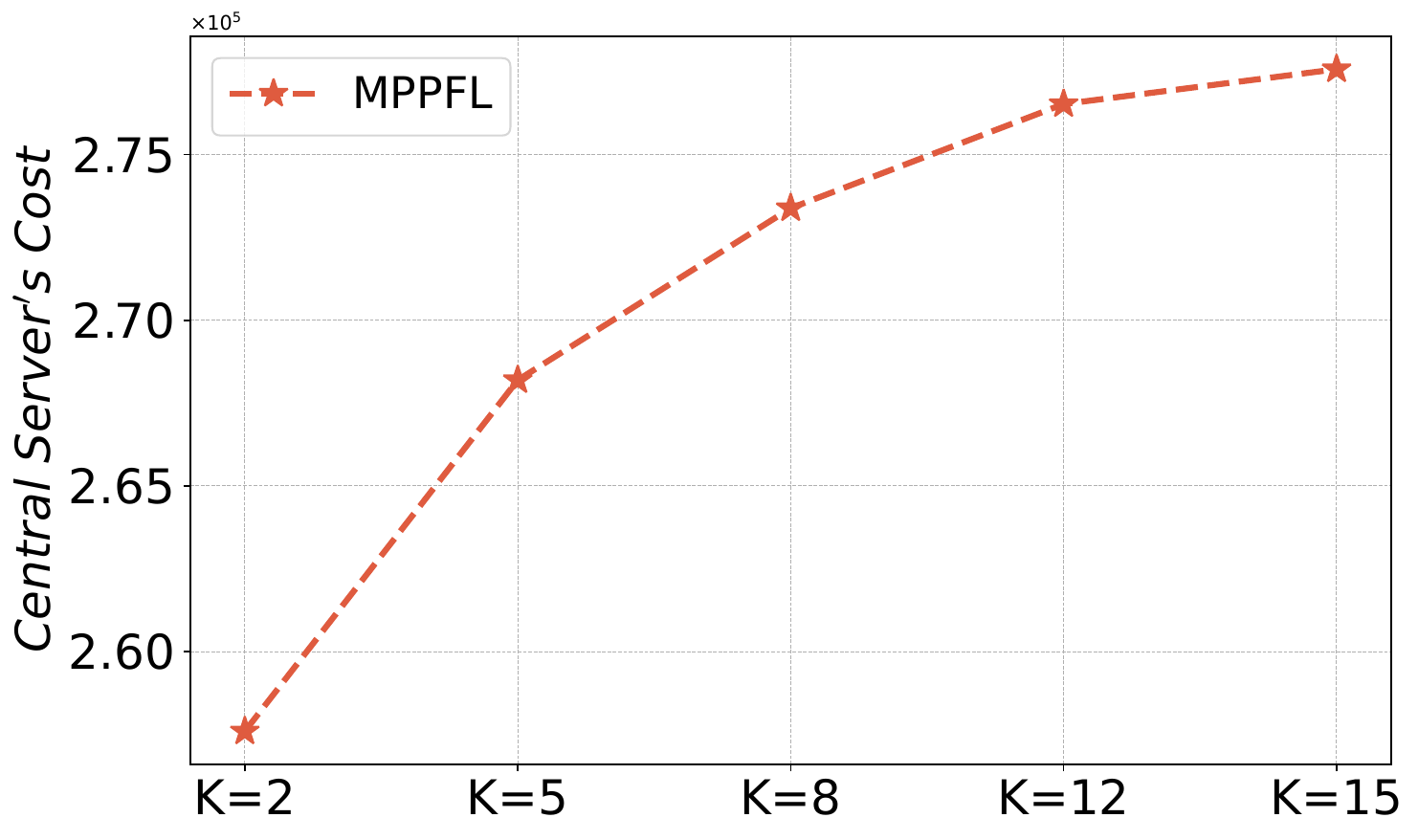}
        \vspace{-20pt}
        \caption{The central server’s cost on CINIC-10 dataset with different $K$.}
        \label{csc_t80_n20_k}
    \end{minipage}
    \hspace{2pt}
    \begin{minipage}{122pt}
        \includegraphics[width=1.0\textwidth, trim=35 40 30 35,clip]{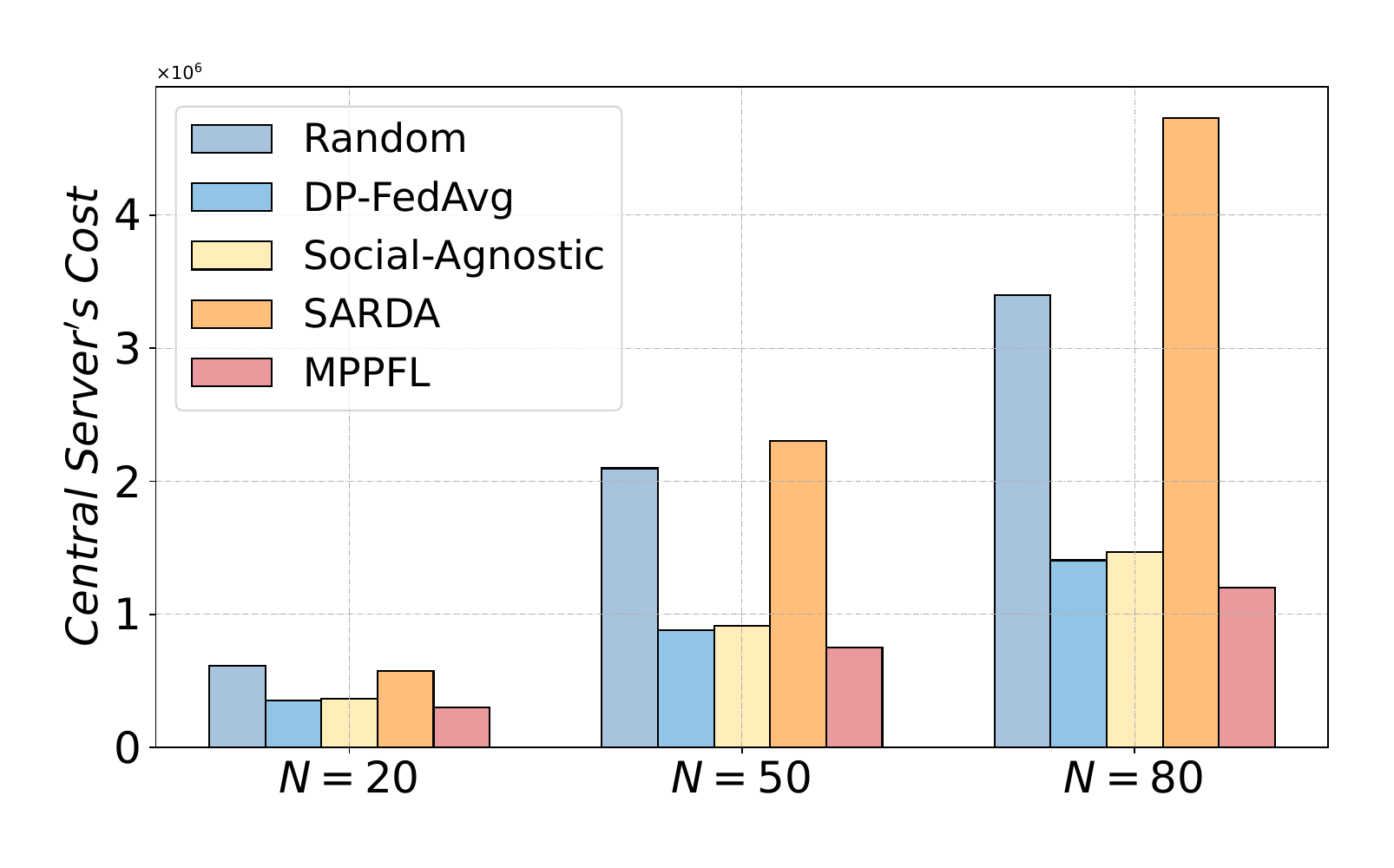}
        \vspace{-20pt}
        \caption{The central server’s cost comparison on CIFAR-10 dataset with different $N$.}
        \label{csc_cinic_cifar10}
    \end{minipage}
    \hspace{2pt}
    \begin{minipage}{122pt}
        \includegraphics[width=1.0\textwidth, trim=35 40 30 35,clip]{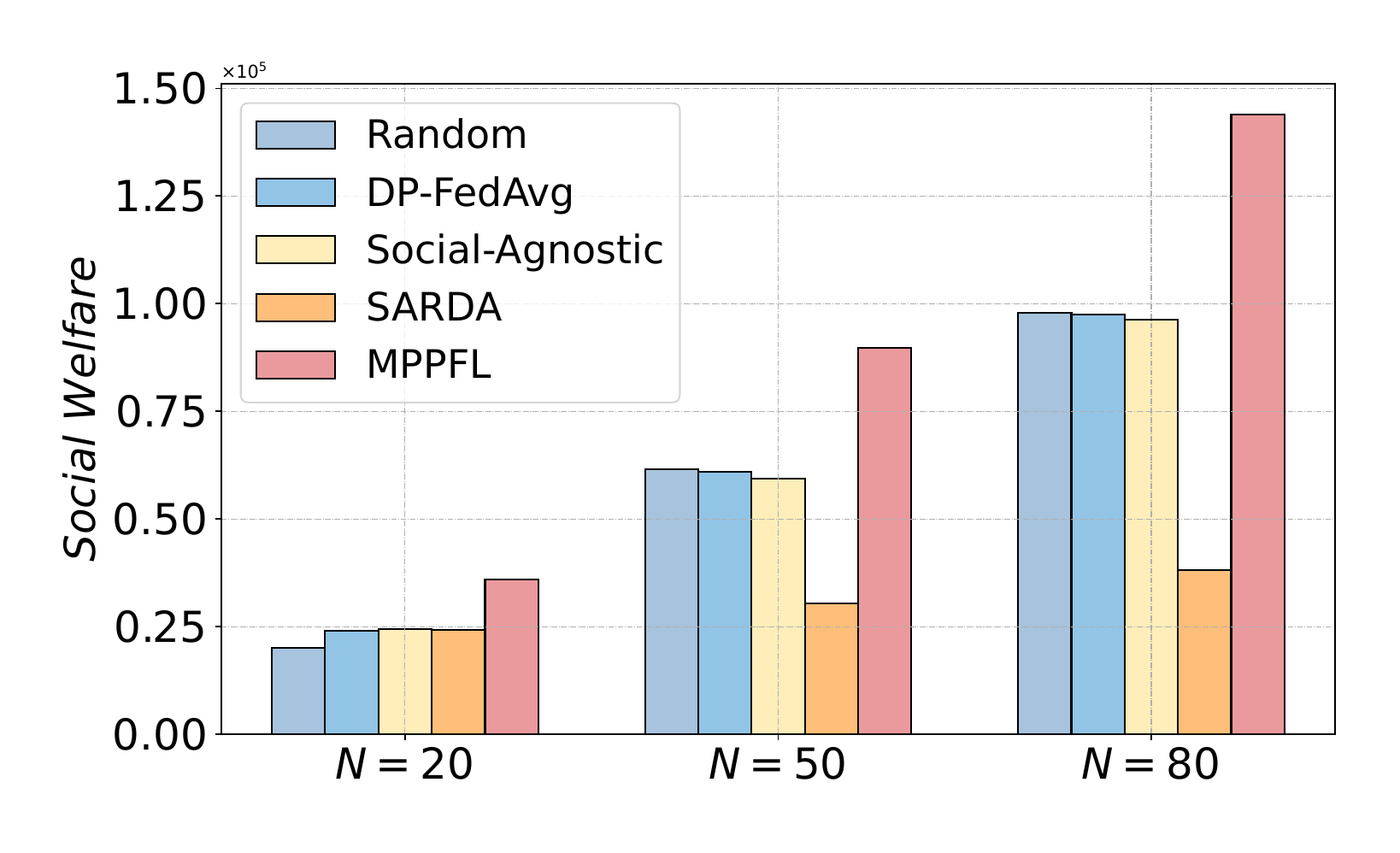}
        \vspace{-20pt}
        \caption{The social welfare comparison on CIFAR-10 dataset with different $N$.}
        \label{sw_cinic_cifar10}
    \end{minipage}
\vspace{-12pt}    
\end{figure*}

\begin{figure*}[t]
\setlength{\abovecaptionskip}{2pt}
\centering
    \begin{minipage}{160pt}
        \includegraphics[width=1.0\textwidth, trim=5 10 5 5,clip]{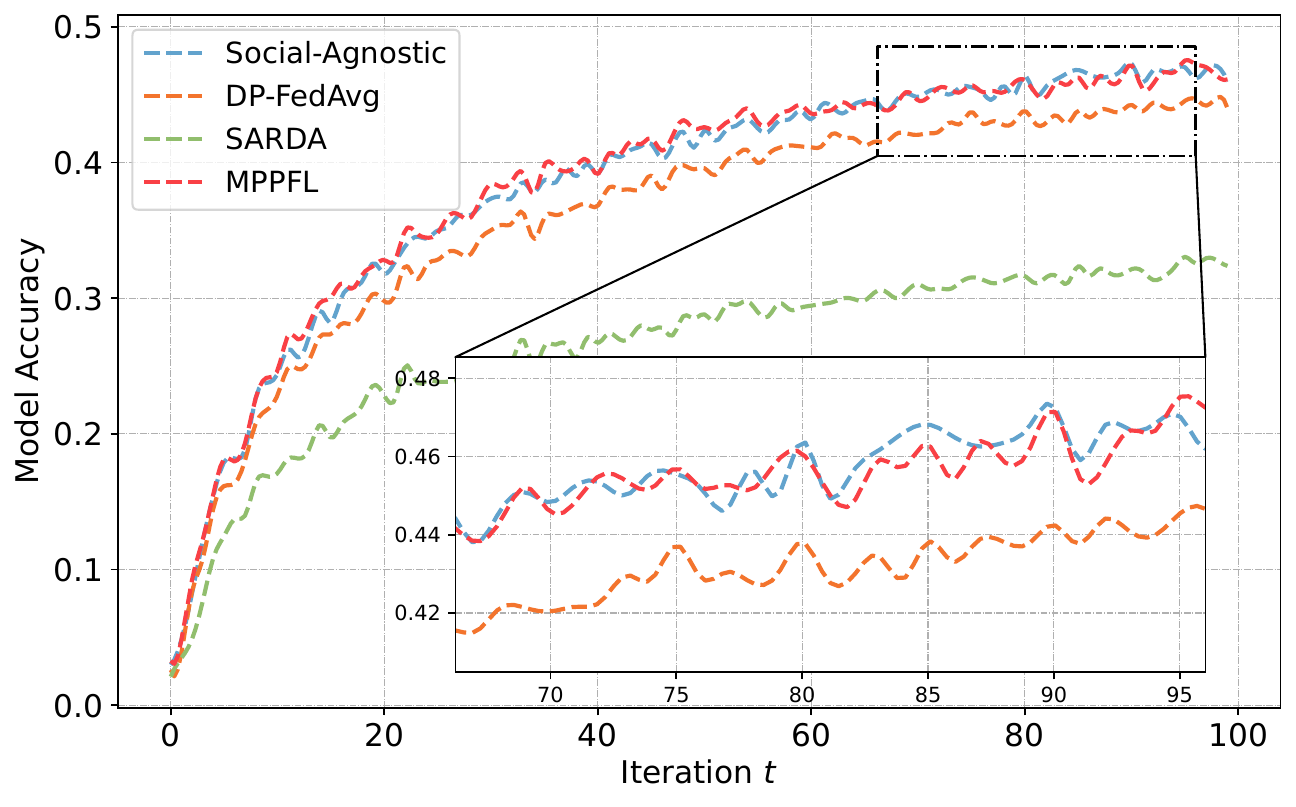}
        \vspace{-20pt}
        \caption{Accuracy curve on CIFAR-100 dataset with non-IID (Dir = 0.3).}
        \label{dir_03_acc_t100n20}
    \end{minipage}
    \hspace{2pt}
    \begin{minipage}{160pt}
        \includegraphics[width=1.0\textwidth, trim=5 10 5 5,clip]{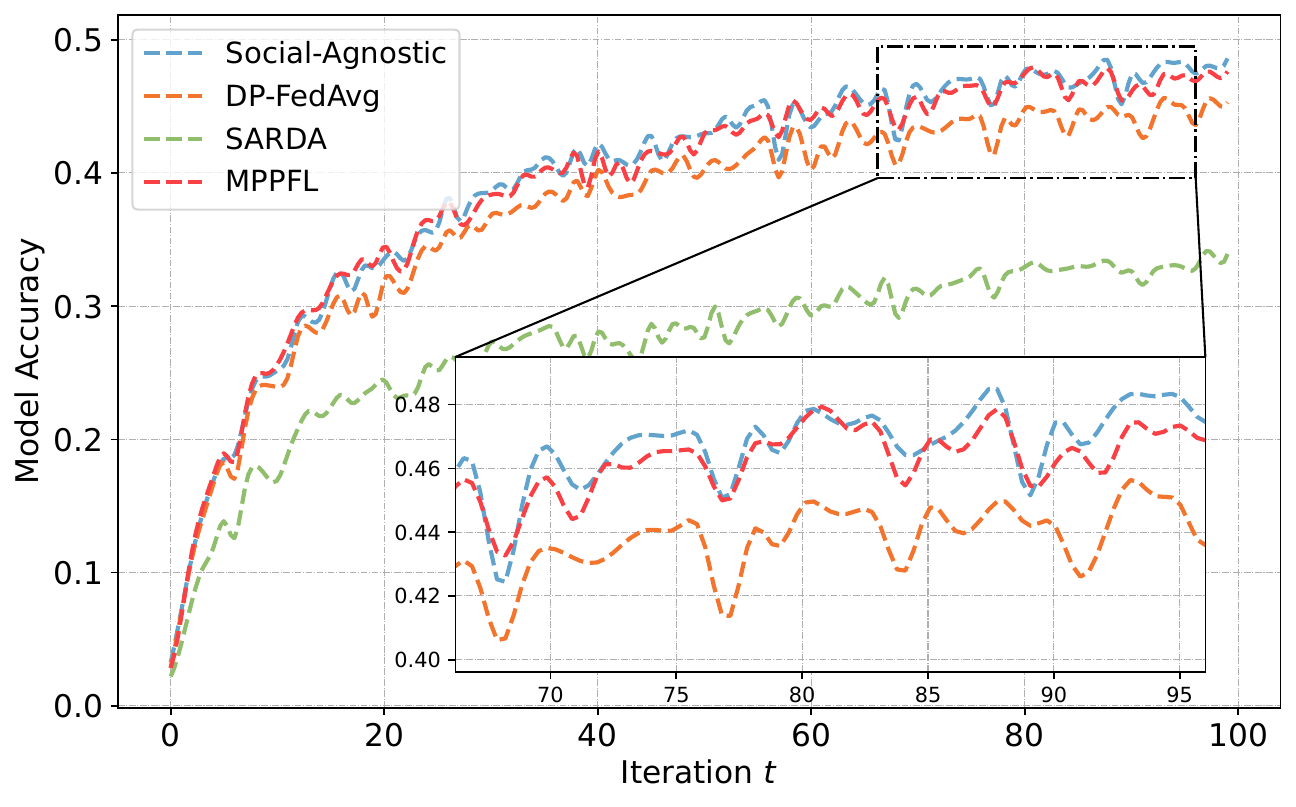}
        \vspace{-20pt}
        \caption{Accuracy curve on CIFAR-100 dataset with non-IID (Dir = 0.6).}
        \label{dir_06_acc_t100n20}
    \end{minipage}
    \hspace{2pt}
    \begin{minipage}{160pt}
        \includegraphics[width=1.0\textwidth, trim=5 10 5 5,clip]{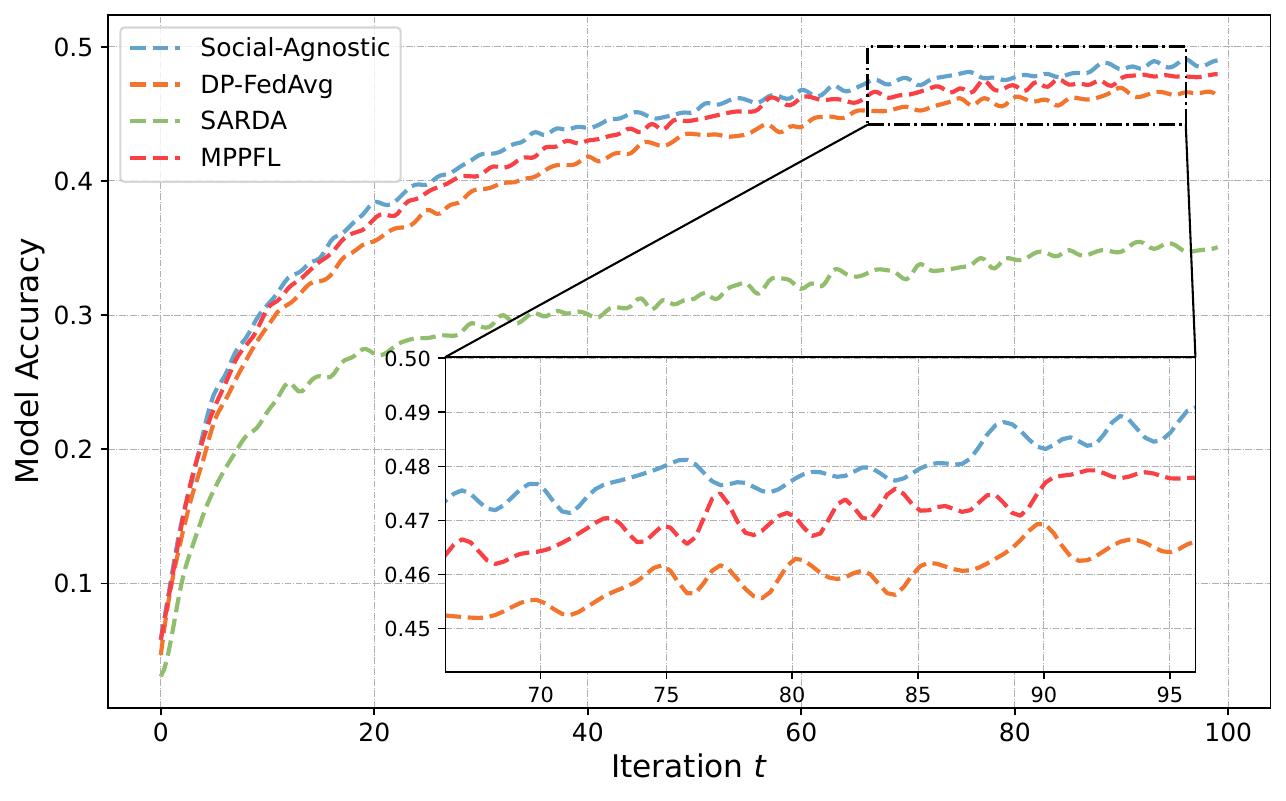}
        \vspace{-20pt}
        \caption{Accuracy curve on CIFAR-100 dataset with IID data.}
        \label{iid_acc_t100n20}
    \end{minipage}
\vspace{-12pt}    
\end{figure*}

\section{Proof of Theorem \ref{theorem_fixed_point}}
\begin{proof}
We begin by examining the external privacy leakage risk coefficient $\sigma_{ij} = \sum\nolimits_{k=1}^K \lambda^{k-1} (\boldsymbol{\widetilde{W}}^k)_{ij}$. Since $\boldsymbol{\widetilde{W}}$ is a row-normalized matrix, each row sums to 1.  By induction, $\boldsymbol{\widetilde{W}}^k$ also remains row-normalized for positive integer $k$. Given that each entry $(\boldsymbol{\widetilde{W}}^k)_{ij}$ is non-negative and bounded by $1$, we derive the upper bound of $\sigma_{ij}$ as:
\begin{align}
    \sigma_{ij} = \sum\nolimits_{k=1}^K \lambda^{k-1} (\boldsymbol{\widetilde{W}}^k)_{ij} \leq \sum\nolimits_{k=1}^K \lambda^{k-1} = \frac{1 - \lambda^K}{1 - \lambda},
\end{align}
where the inequality follows from the geometric series summation. This bound ensures that $\sigma_{ij}$ decays exponentially with the propagation length $K$, regulated by the decay factor $\lambda$, thereby limiting the multi-hop influence client $x_j$ can exert on client $x_i$.

Based on Definition \ref{mean_field} and Theorem \ref{optimal_rho}, for each client $x_i, i\in \{1, 2, \ldots, N\}$, we substitute the mean-field esimator $\phi_{i}(t) = \frac{1}{N} \sum\nolimits_{j = 1}^{N} \sigma_{ij} \times \rho_{j}(t)$ into Eq.~(\ref{optimal_privacy_budget}). Then, the expression of the privacy budget of client $x_i, \ i \!\in\! \{1, 2, \ldots, N\}$ at iteration $t \!\in\! \{0,1, \ldots, T\}$ can be further reformulated as follows:
\begin{align}\label{optimal_rho_reformulated}
    \rho_{i}(t) = \frac{r(t) - b_{i}}{2a_{i}} - \alpha \sum\nolimits_{j = 1}^{N} \sigma_{ij} \times \rho_{j}(t).
\end{align}

From Eq.~(\ref{optimal_rho_reformulated}), we notice that the privacy budget $\rho_{i}(t)$ of client $x_i, \ i \!\in\! \{1, 2, \ldots, N\}$ at $t$-th iteration with $t \!\in\! \{0,1, \ldots, T\}$ is a continuous function of the entire set $\{ \rho_{i}(t)|t\in \{0,1,\ldots,T\} ,i\in \{1, 2,\ldots,N\} \}$ of all clients over time. We define a continuous mapping from $\{ \rho_{i}(t)|t\in \{0,1,\ldots,T\},i\in \{1,2,\ldots,N\} \}$ to client $x_i$'s privacy budget $\rho_{i}(t)$ at $t$-th iteration by:
\begin{align}\label{mapping}
\Theta_{i}^{t}(\{ \rho_{i}(t)|t\in \{0,1,\ldots,T\} ,i\in \{1,2,\ldots,N\} \}) \!=\! \rho_{i}(t). \!
\end{align}

To summarize any possible mapping $\Theta_{i}(t)$, we define the following vector functions as a mapping from $\{\rho_{i}(t)|t \!\in\! \{0, 1, \ldots, T\}, i \!\in\! \{1, 2, \ldots, N\} \}$ to privacy budget set of all clients at any iteration $t$:
\begin{align} \label{mapping_all_clients_iterations}
&\Theta(\{ \rho_{i}(t)|t\in \{0,1,\ldots,T\} ,i\in \{1,2,\ldots,N\} \})  \nonumber \\
=&(\Theta_{1}^{0}(\{ \rho_{i}(t)|t\in \{0,1,\ldots,T\} ,i\in \{1,2,\ldots,N\} \}),\ldots, \nonumber \\
&(\Theta_{1}^{T}(\{ \rho_{i}(t)|t\in \{0,1,\ldots,T\} ,i\in \{1,2,\ldots,N\} \}),\ldots, \nonumber \\
&(\Theta_{N}^{0}(\{ \rho_{i}(t)|t\in \{0,1,\ldots,T\} ,i\in \{1,2,\ldots,N\} \}),\ldots, \nonumber \\
&(\Theta_{N}^{T}(\{ \rho_{i}(t)|t\in \{0,1,\ldots,T\} ,i\in \{1,2,\ldots,N\} \})).
\end{align}

Thus, based on Eq.~(\ref{mapping_all_clients_iterations}), the fixed point of mapping $\Theta(\{ \rho_{i}(t)|t \in \{0,1,\ldots,T\} ,i\in \{1,2,\ldots,N\} \})$ should be reached to make $\phi_{i}(t)$ replicate $\frac{1}{N} \sum\nolimits_{j = 1}^{N} \sigma_{ij} \times \rho_{j}(t)$. Assume that a boundary condition of privacy budget $\rho_{i}(t)$ exists, i.e., $\rho_{l} \le \rho_{i}(t) \le \rho_{h}$ with $\rho_{l} = \frac{m_l-\alpha S m_h}{1-\alpha^{2} S^{2}}$ and $\rho_{h} = \frac{m_h-\alpha S m_l}{1-\alpha^{2} S^{2}}$ for $t\in \{0,1,\ldots,T\} ,i\in \{1,2,\ldots,N\}$, where $m_l = \min_{i,t}\frac{r(t) - b_{i}}{2a_{i}},m_h = \max_{i,t}\frac{r(t) - b_{i}}{2a_{i}}$, $S = \frac{1-\lambda^K}{1-\lambda}$ and $\alpha S<1$, thereby, we define a continuous space $\Omega = [\rho_{l}, \rho_{h}]^{N\times(T+1)}$ for $\rho_{i}(t)$, which is a compact convex set in $\mathbb{R}^{N\times (T+1)}$. We can verify that if $\rho_j(t)\in[\rho_l,\rho_h]$ for all $j$, the right-hand side of Eq.~(\ref{optimal_rho_reformulated}) also lies within $[\rho_l,\rho_h]$. From Eq.~(\ref{mapping}), each mapping $\Theta_{i}^{t}$ is continuous in the continuous space $\Omega$. Consequently, $\Theta$ is a continuous self-mapping on the compact convex space from $\Omega$ to $\Omega$. By Brouwer’s Fixed-Point Theorem, $\Theta$ has a fixed point in $\Omega$.   

Based on the above formulation, we have defined the Complete Metric Space $\Omega$ over the set of all privacy budget parameters $\boldsymbol{\rho}$. To quantify the divergence between any two variables in this space, we introduce a distance function defined by the $\ell_2$-norm, i.e., $\Gamma = \|\rho_{i}(t) - \rho_{m}(t)\|_{2}^{2}$, and have $\Gamma(\Theta_{i}^{t}(\rho_{i}(t)), \Theta_{m}^{t}(\rho_{m}(t))) \leq \nu \| \rho_i(t) - \rho_m(t) \|_2^2$, where $\nu = \alpha^2 S^2$ satisfies $0 \leq \nu <1$, indicating that the mapping $\Theta_{i}^{t}$ is a contraction mapping on $\Omega$ according to Banach’s Fixed Point Theorem, which guarantees the existence and uniqueness of the fixed point of $\Theta_{i}^{t}$, denoted by $\rho^{*}_{i}(t)$. Moreover, the contraction property of $\Theta_{i}^{t}$ implies that the discrepancy between successive elements in the series is reduced by a multiplicative factor $\nu$ at each iteration. Specifically, for any given $\mathcal{J}$, there exists an integer $\mathcal{H} > 0$ such that for all $u_1,u_2 > \mathcal{H}$, the inequality $\Gamma(\Theta_{u_1}^{t}(\rho_{u_1}(t)),\Theta_{u_2}^{t}(\rho_{u_2}(t))) < \mathcal{J}$ holds. Thus, the sequence generated by $\Theta_{i}^{t}$ forms a Cauchy sequence and converges to the unique fixed point $\rho^{*}_{i}(t)$.   \end{proof}

\begin{table}[t]
\renewcommand\arraystretch{1.0}
\caption{Basic Information of Datasets.}
\begin{center}
\resizebox{0.48\textwidth}{!}{\begin{tabular}{c|c|c|c|c}
\toprule[1pt]
\textbf{Datasets}&\textbf{Traing Set Size}&\textbf{Test Set Size}&\textbf{Class} &\textbf{Image Size}  \\ \cmidrule[0.5pt](l{1pt}r{0pt}){1-5}

\textbf{FMNIST} & 60,000 & 10,000  & 10 & 1 $\times$ 28 $\times$ 28  \\ \cmidrule[0.5pt](l{1pt}r{0pt}){1-5}

\textbf{CIFAR-10} & 50,000 & 10,000  & 10 & 3 $\times$ 32 $\times$ 32  \\ \cmidrule[0.5pt](l{1pt}r{0pt}){1-5}

\textbf{SVHN} & 73,257 & 26,032  & 10 & 3 $\times$ 32 $\times$ 32  \\ \cmidrule[0.5pt](l{1pt}r{0pt}){1-5}

\textbf{CIFAR-100} & 50,000 & 10,000  & 100 & 3 $\times$ 32 $\times$ 32  \\ \cmidrule[0.5pt](l{1pt}r{0pt}){1-5}

\textbf{CINIC-10} & 90,000 & 90,000  & 10 & 3 $\times$ 32 $\times$ 32  \\ 
\bottomrule[1pt]
\end{tabular}}
\label{basic_information_of_datasets}
\end{center}
\end{table}

\begin{table}[t]
\setlength{\abovecaptionskip}{0pt} 
\caption{FL Training Hyperparameter Settings.}
\renewcommand\arraystretch{1.0}
\begin{center}
\resizebox{0.48\textwidth}{!}{\begin{tabular}{c|c|c|c|c}
\toprule[1pt]
\textbf{Datasets} & \textbf{Learning Rate}&\textbf{Batchsize} & \textbf{Global Iteration} & \textbf{Local Epoch}  \\ \cmidrule[0.5pt](l{1pt}r{0pt}){1-5}

\textbf{FMNIST} & 0.01 & 32  & 30 & 5 \\ \cmidrule[0.5pt](l{1pt}r{0pt}){1-5}

\textbf{CIFAR-10} & 0.01 & 32  & 80 & 5 \\ \cmidrule[0.5pt](l{1pt}r{0pt}){1-5}

\textbf{SVHN} & 0.01 & 256  & 120 & 5 \\ \cmidrule[0.5pt](l{1pt}r{0pt}){1-5}

\textbf{CIFAR-100} & 0.01 & 64  & 100 & 5 \\ \cmidrule[0.5pt](l{1pt}r{0pt}){1-5}

\textbf{CINIC-10} & 0.01 & 64  & 60 & 5 \\ 
\bottomrule[1pt]
\end{tabular}}
\label{hyperparameter}
\end{center}
\vspace{-6pt}
\end{table}

\section{Additional Experiment Results}\label{other_Results}

Table \ref{basic_information_of_datasets} summarizes the basic characteristics of the five datasets used in our experiments, including the number of training and test samples, the number of classes, and the input image size. The selected datasets span diverse domains and difficulty levels: FMNIST is a grayscale dataset of fashion images with 10 categories, while CIFAR-10, CIFAR-100, SVHN, and CINIC-10 are RGB datasets with varying class sizes and sample distributions. Notably, CINIC-10 contains a larger and more balanced training and test split, making it suitable for scalability studies. Table \ref{hyperparameter} presents the federated training hyperparameter settings for each dataset. 

Figs.~\ref{csc_t80_n20_alpha} and~\ref{csc_t80_n20_k} present the central server’s cost on the CINIC-10 dataset under different values of $\alpha$ and the longest propagation length $K$, respectively. As shown in Figure \ref {csc_t80_n20_alpha}, increasing $\alpha$, which reflects clients’ sensitivity to external privacy risks, leads to a growing server cost. This is because more risk-averse clients require stronger incentives to participate, resulting in higher payments from the server. In Figure \ref {csc_t80_n20_k}, a similar trend is observed as $K$ increases, indicating that longer propagation paths amplify the accumulation of external risks, thereby driving up the overall cost. These results highlight that both client sensitivity and the structure of the social network substantially influence the server’s incentive expenditure. We also compare the utilities of different numbers of clients on the CIFAR-10 dataset, as shown in Figs. \ref{csc_cinic_cifar10} and \ref{sw_cinic_cifar10}. It serves as a supplementary experiment to support the effectiveness of MPPFL in controlling the central server's cost and coordinating client behavior efficiency with different participating client sizes.

Figs. \ref{dir_03_acc_t100n20}–\ref{iid_acc_t100n20} present the accuracy performance of various privacy-preserving mechanisms on the CIFAR-100 dataset under different levels of data heterogeneity. As the Dirichlet parameter increases from $0.3$ to $0.6$, and eventually to the IID setting, the accuracy of MPPFL remains consistently close to that of Social-Agnostic. In contrast, DP-FedAvg and SARDA exhibit lower performance across all settings, primarily due to their failure to account for external privacy risks and the absence of an adaptive privacy budget design.

\end{document}